\documentclass[11pt]{article}

\usepackage[a4paper,margin=30mm]{geometry}
\usepackage{amsmath,amssymb,amsthm,mathtools}
\usepackage{microtype}
\usepackage[authoryear,round]{natbib}
\usepackage{authblk}
\usepackage[hidelinks]{hyperref}
\providecommand{\doi}[1]{%
  \href{https://doi.org/#1}{\nolinkurl{doi:#1}}}

\numberwithin{equation}{section}
\allowdisplaybreaks
\newtheorem{theorem}{Theorem}[section]
\newtheorem{lemma}[theorem]{Lemma}
\newtheorem{proposition}[theorem]{Proposition}
\newtheorem{corollary}[theorem]{Corollary}
\theoremstyle{definition}
\newtheorem{definition}[theorem]{Definition}
\theoremstyle{remark}
\newtheorem{remark}[theorem]{Remark}

\newcommand{\R}{\mathbb R}
\newcommand{\X}{\mathcal X}
\newcommand{\Hk}{\mathcal H_k}
\newcommand{\EI}{\operatorname{EI}}
\newcommand{\dist}{\operatorname{dist}}

\title{Simple-regret rates and minimax optimality of  \\
  fixed-prior expected improvement in \\ Mat\'ern and squared-exponential RKHSs}

 \author[1]{Emmanuel Vazquez}

\author[2]{S\'ebastien Petit} \affil[1]{Universit\'e Paris-Saclay,
  CNRS, CentraleSup\'elec,
  Laboratoire des Signaux et Syst\`emes, 91190 Gif-sur-Yvette, France\\
  \href{mailto:emmanuel.vazquez@centralesupelec.fr}%
  {emmanuel.vazquez@centralesupelec.fr}} \affil[2]{Laboratoire
  national de m\'etrologie et d'essais (LNE),
  1 rue Gaston Boissier, 75724 Paris Cedex 15, France\\
  \href{mailto:sebastien.petit@lne.fr}{sebastien.petit@lne.fr}}

\date{}

\hypersetup{ pdftitle={Simple-regret rates and minimax optimality of fixed-prior
  expected improvement in Mat\'ern and squared-exponential RKHSs},
  pdfauthor={Emmanuel Vazquez; S\'ebastien Petit},
  pdfsubject={Convergence analysis of expected improvement},
  pdfkeywords={Bayesian optimization, expected improvement, RKHS,
    posterior standard deviation, Kolmogorov widths} }

\begin{document}
\maketitle

\begin{abstract}
We study the expected improvement (EI) policy for minimizing a deterministic
objective function $f$ on a nonempty compact set $\X\subset\mathbb R^d$.
We assume that $f$ belongs to the RKHS $\Hk$ of a continuous
positive-semidefinite kernel $k$ on $\X$.  Function values are observed
exactly, and EI is computed
from a fixed zero-mean Gaussian-process model with covariance $\sigma^2k$.
After an initial design, the policy queries a point whose EI is at least
a fixed positive fraction of its maximum.

We identify the normalized posterior standard deviation at a candidate point
$x$ with the norm of the corresponding innovation in the canonical feature
space, namely the component of $k(x,\cdot)$ orthogonal to the span of the
preceding evaluation representers.  Sequential separation radii bound the
ranked innovation norms along arbitrary query sequences.  We estimate these
radii using Gram determinants and Kolmogorov widths for subspaces of
different dimensions, then combine the estimates with a one-step regret
inequality to obtain finite-budget bounds for simple regret.

After $N$ post-initial queries, simple regret is $O(N^{-\nu/d})$ for
isotropic Mat\'ern kernels of smoothness $\nu>0$.  For the isotropic
squared-exponential kernel, simple regret is
$O(\exp[-c_1\min\{N,\allowbreak N^{1/d}\log(eN)\}])$ for some $c_1>0$.
With exact EI maximization, it is
$O(\exp[-c_2N^{1/d}\allowbreak\log(eN)])$ for some $c_2>0$.  For every fixed
$B\geq0$, these bounds are uniform over the RKHS ball of radius $B$.  The
zero prior mean may be replaced by a known
$\mu_0\in\mathcal C(\X)$ when
$f-\mu_0\in\Hk$.

If $\X$ has nonempty interior and $B>0$, exact~EI is minimax-rate optimal
over the radius-$B$ RKHS ball for Mat\'ern kernels, even among randomized
strategies whose final recommendation need not be a query point.  For the squared-exponential
kernel, it is minimax-rate optimal up to constants in the exponent among
deterministic methods whose final recommendation may be any point of $\X$.
\end{abstract}


\section{Introduction}
\label{sec:introduction}

Let $\X\subset\mathbb R^d$, $d\ge 1$, be nonempty and compact, and let
$k:\X\times\X\to\mathbb R$ be continuous and positive semidefinite.
Write $\Hk$ for the reproducing-kernel Hilbert space (RKHS) of $k$.  We seek to
minimize a deterministic function $f\in\Hk$ satisfying
$\lVert f\rVert_{\Hk}\leq B$, where $B\geq0$.  Function values are
observed exactly.

Starting from an initial design of $n_0\geq1$ query points
$x_1,\ldots,x_{n_0}$, we use expected improvement (EI) as a sampling
criterion to select $x_{n_0+1},x_{n_0+2},\ldots$  To define EI, we
use a fixed zero-mean Gaussian-process model $Z$ with covariance
$\sigma^2k$, $\sigma>0$.  Given the past query points
$x_1,\ldots,x_n$, $n\geq n_0$, and their observed values
$f(x_1),\ldots,f(x_n)$, we condition the model on $Z(x_i)=f(x_i)$,
$1\leq i\leq n$.  For each candidate point $x\in\X$, the posterior
distribution of $Z(x)$ serves as a predictive distribution for $f(x)$.
The EI at $x$ is the posterior expected amount by which $Z(x)$ falls
below the best observed value.  The (exact)~EI policy selects a point
$x_{n+1}$ at which EI is maximal.  A weak-EI policy
selects a point whose EI is at least a fixed fraction $\eta\in(0,1]$
of the maximum.  We evaluate the objective function at $x_{n+1}$ and
condition the model on the enlarged data.  We call the resulting sequential
design $(x_n)_{n\geq1}$ of query points a fixed-prior EI trajectory.

With one evaluation remaining and the final estimate of the minimizer
required to be one of the query points, querying an EI maximizer is a
one-step Bayes-optimal strategy: it minimizes the posterior
expectation of the best value observed after that evaluation
\citep{vazquezBect2010}.  Over a longer horizon, repeating this
one-step rule yields a myopic policy that nevertheless often performs well
in practice and is widely used both in machine learning and in engineering
design based on numerical simulations.

\paragraph{Historical background.}
Early probabilistic methods for sequential optimization appeared in the 1960s
and 1970s \citep{matheronFormery1962,kushner1964,mockus1975}, and the notion
of EI emerged during the same period \citep{mockus1978}.
\citet[Chapter~12]{garnett2023} gives a historical account.
\citet{jones1998} popularized EI for computer experiments through the efficient
global optimization algorithm.  EI has since become a standard sampling
criterion in what is now commonly called Bayesian optimization.

\paragraph{Previous results and minimax comparisons.}
When EI is maximized exactly and~$k$ has the no-empty-ball property
(Definition~\ref{def:neb}), the query sequence is dense in $\X$ for every
$f\in\Hk$, and the best observed value converges to the global minimum
\citep[Theorem~6]{vazquezBect2010}.
(For an objective function sampled from the Gaussian-process prior,
\citet[Theorem~7]{vazquezBect2010} also prove almost-sure
density for the exact~EI policy under the no-empty-ball property, but this
probabilistic result lies outside the deterministic RKHS setting considered
here.)

Let $T>n_0$ be the total number of evaluations and put $N:=T-n_0$.
Define the simple regret after $T$ evaluations by
$$
  r_T:=\min_{1\leq i\leq T}f(x_i)-\min_{x\in\X}f(x)\,.
$$
For exact~EI strategies whose initial design is chosen independently of $f$,
possibly at random, and with arbitrary selection among EI maximizers,
\citet[Theorem~2]{bull2011} proves worst-case expected simple-regret bounds
when the covariance parameters are fixed and an unknown constant mean is
integrated under a flat prior.  If $\X$ has nonempty
interior and the kernel is Mat\'ern with smoothness~$\nu$, the upper bound is
$O(N^{-\nu/d})$ when
$0<\nu<1$, $O(N^{-1/d}\sqrt{\log N})$ when $\nu=1$, and $O(N^{-1/d})$
when $\nu>1$.  Thus its polynomial exponent saturates at $1/d$ for
$\nu\geq1$.  For the squared-exponential kernel, the upper bound is
$O(N^{-1/d})$.

\citet[Theorem~1]{bull2011} establishes the Mat\'ern minimax rate
$N^{-\nu/d}$.  The randomized $\varepsilon$-greedy EI policy defined by
\citet[Definition~4]{bull2011} intersperses exact~EI queries with uniformly
random queries on $\X$.  \citet[Theorem~5]{bull2011} proves that this policy
attains the minimax rate, up to logarithmic factors, for every finite $\nu$.
The contrast between the saturation of the upper bound established by Bull for
the exact~EI policy and the rate attained by the randomized
$\varepsilon$-greedy EI policy raises the natural question whether the
exact fixed-prior EI policy attains the rate $N^{-\nu/d}$ for every
$\nu>0$, without space-filling random queries.  We prove in
Theorem~\ref{thm:matern-minimax} that it does.

\paragraph{Related work.}
With noisy observations and an objective function in an RKHS,
\citet[Theorems~1 and~2]{tranThe2022} obtain high-probability bounds on
cumulative regret evaluated at the points recommended by two modified EI
policies, one with an iteration-dependent multiplier of the predictive
standard deviation and the other with a multiplier chosen from the total
budget.  When the objective function is sampled from a Gaussian process,
\citet[Theorem~4.9]{wangNoisy2025} bound with high probability the difference
between the best noisy observation and the global minimum for EI with Gaussian
observation noise.

For the noise-free Gaussian-process upper confidence bound (GP-UCB),
\citet[Lemma~3 and Theorem~2]{iwazaki2025} gives the simple-regret upper bounds
$O(N^{-\nu/d}(\log N)^{\nu/d})$ for Mat\'ern kernels with $\nu>1/2$ and
$O(\sqrt{N}\exp\{-cN^{1/(d+1)}\})$ for the squared-exponential kernel, for
some $c>0$.  The polynomial exponent in the Mat\'ern bound agrees with the
exponent in Theorem~\ref{thm:matern-rate} in this work.  The two GP-UCB bounds follow
from estimates of the maximum information gain for a Gaussian-process model
with an auxiliary noise variance.  The information-gain estimate for the
Mat\'ern kernel assumes uniformly bounded kernel eigenfunctions.
Our proof draws on methods from
greedy approximation and scattered-data approximation in RKHSs.

\paragraph{Main results.}
For an isotropic Mat\'ern kernel of smoothness
$\nu>0$, we show in Theorem~\ref{thm:matern-rate} that every fixed-prior~EI
trajectory generated by a weak-EI policy has simple regret
$$
  r_T=O\bigl(N^{-\nu/d}\bigr).
$$
When $B>0$ and $\X$ has nonempty interior,
Theorem~\ref{thm:matern-minimax} establishes that the exact~EI policy is
minimax-rate optimal among deterministic methods whose final recommendation
may be any point of $\X$.

For the isotropic squared-exponential kernel,
Theorem~\ref{thm:se-rate} shows that, along every fixed-prior~EI trajectory
generated by a weak-EI policy, the simple regret
satisfies, for some $A',b'>0$,
$$
  r_T
  \leq
  A'\exp\{-b'\min\{
    N,\,
    N^{1/d}\log(eN)
  \}\},
$$
for all sufficiently large~$N$.  When $B$, $\X$, $k$, $\sigma$, $\eta$, and
the initial design are fixed, both kernel bounds hold uniformly over
$f\in\Hk$ with $\lVert f\rVert_{\Hk}\leq B$.  When $B>0$ and $\X$ has
nonempty interior, Theorems~\ref{thm:se-rate}
and~\ref{thm:se-minimax} establish that the exact~EI policy is minimax-rate
optimal up to constants in the exponent among deterministic methods whose
final recommendation may be any point of $\X$.

\paragraph{Regret bounds from approximation estimates.}
We call the component of
$k(x_{n+1},\cdot)$ orthogonal to the span of the preceding evaluation
representers $k(x_1,\cdot),\ldots,k(x_n,\cdot)$ the selected-point innovation
and denote its norm by $v_{n+1}$.
Proposition~\ref{prop:one-step-regret} and regret monotonicity give,
for some constants $0\leq q<1$ and $C>0$, the one-step regret bound
\begin{equation}
  r_{n+1}\leq\min\{r_n,\;q r_n+C v_{n+1}\},
  \qquad n\geq n_0.
  \label{eq:intro-one-step-regret}
\end{equation}
The regret analysis combines~\eqref{eq:intro-one-step-regret} with approximation
bounds for the selected-point innovation norms.
The sequential separation radius of a kernel $k$ on $\X$, denoted by
$\delta_m(k,\X)$ for $m\geq1$, is the supremum over ordered $m$-tuples of
the smallest successive innovation norm.
Let $k_0$ be the normalized posterior covariance kernel after conditioning
on the initial design.  Let
$$
  v_{[1]}\geq\cdots\geq v_{[N]}
$$
be the decreasing rearrangement of
$v_{n_0+1},\ldots,v_T$.
Lemma~\ref{lem:innovation-order-statistics} gives
$$
  v_{[m]}
  \leq \delta_m(k_0,\X)
  \leq \delta_m(k,\X),
  \qquad 1\leq m\leq N.
$$
For each $1\leq L\leq N$, we use $r_{n+1}\leq q r_n+C v_{n+1}$
from~\eqref{eq:intro-one-step-regret} at the $L$ evaluations with the
smallest realized innovation norms, and regret monotonicity at all remaining
iterations.  Theorem~\ref{thm:order-statistic-transfer} expresses the
resulting finite-budget estimate in terms of the separation radii.

To bound the separation radii, we use the Kolmogorov widths
$d_m(k,\X)$, which measure how well the evaluation representers can be
approximated uniformly by subspaces of $\Hk$ of dimension at most $m$.  The
connection between sequential separation radii and these widths draws on
results from greedy approximation
\citep{binev2011,devore2013,wenzelSantinHaasdonk2023}.  Given
$x_1,\ldots,x_m\in\X$, consider the kernel matrix
$(k(x_i,x_j))_{1\leq i,j\leq m}$.  Lemma~\ref{lem:width-product} bounds its
$j$th largest eigenvalue in terms of a Kolmogorov width whose subspace
dimension is chosen separately for each $j$.
Standard power-function estimates from scattered data approximation, applied
to regular grids, give
$d_m(k,\X)=O(m^{-\nu/d})$ for Mat\'ern kernels
\citep[Theorem~3.2]{schabackWendland2002} and
$d_m(k,\X)=O(\exp[-c m^{1/d}\log m])$ for the squared-exponential kernel,
for some $c>0$
\citep[Theorem~11.22]{wendland2005}.
Lemma~\ref{lem:width-product} converts these width estimates into bounds on
the separation radii.

\paragraph{Scope and organization.}
Weak-EI policies attain the Mat\'ern rate and the squared-exponential
bound~\eqref{eq:se-rate}, are simple-regret consistent, and have dense query
points when~$k$ has the no-empty-ball property.  Exact EI maximization is
needed for the minimax comparisons, while the analysis assumes exact
observations and covariance parameters chosen independently of the data.
Appendix~\ref{app:known-means} extends the upper regret bounds to a known
continuous prior mean.

Sections~\ref{sec:setting}--\ref{sec:width-geometry} establish the one-step
regret bound and approximation bounds used in Section~\ref{sec:kernels} to
prove the rate theorems.  Section~\ref{sec:minimax-comparison} contains the
minimax comparisons, and Section~\ref{sec:scope} discusses their scope.

\section{Setting and preliminary results}
\label{sec:setting}

\subsection{Fixed-prior model and feature-space innovations}

Let $\X\subset\R^d$, $d\geq1$, be nonempty and compact, and let
$k:\X\times\X\to\mathbb R$ be continuous and positive semidefinite.
Write $\Hk$ for its reproducing-kernel Hilbert space.  We use
\begin{equation}
  Z\sim\operatorname{GP}(0,\sigma^2 k),\qquad \sigma>0,
  \label{eq:fixed-prior-model}
\end{equation}
as a fixed prior model.

Let $B\geq0$ and assume that the deterministic objective function satisfies
\begin{equation}
  f\in\Hk,\qquad \lVert f\rVert_{\Hk}\leq B.
  \label{eq:rkhs-ball}
\end{equation}
Without loss of generality, we assume
\begin{equation}
  k(x,x)\leq 1,\qquad x\in\X.
  \label{eq:normalization}
\end{equation}
The initial design consists of $n_0\geq1$ query points.
At every query point $x_i$, we observe
$z_i:=f(x_i)$ exactly.

Define, for $j\geq1$ and $\mathbf{x}=(x_1,\ldots,x_j)\in\X^j$,
$$
  V_{\mathbf{x}}:=\operatorname{span}\{k(\cdot,x_i),\,1\leq i\leq j\},
$$
and set $V_{\varnothing}:=\{0\}$.  Let $\Pi_{\mathbf{x}}$ be the orthogonal
projector onto~$V_{\mathbf{x}}$ for every finite tuple $\mathbf{x}$.  The
orthogonal projector onto~$V_{\mathbf{x}}^\perp$ is
$$
  \Pi_{\mathbf{x}}^\perp:=I-\Pi_{\mathbf{x}}.
$$
Define, for $x\in\X$,
\begin{equation}
  \iota_{\mathbf{x}}(x):=\Pi_{\mathbf{x}}^\perp k(\cdot,x),
  \qquad
  P_{\mathbf{x}}(x):=\lVert\iota_{\mathbf{x}}(x)\rVert_{\Hk}
  =\dist_{\Hk}\bigl(k(\cdot,x),V_{\mathbf{x}}\bigr).
  \label{eq:tuple-power-function}
\end{equation}
The vector $\iota_{\mathbf{x}}(x)$ is the innovation: the
component of $k(\cdot,x)$ orthogonal to~$V_{\mathbf{x}}$.  The function
$P_{\mathbf{x}}$ is the power function for $\mathbf{x}$
\citep{wendland2005}.  Its value $P_{\mathbf{x}}(x)$ is the innovation norm
at~$x$.

For the query history $\mathbf{x}=(x_1,\ldots,x_n)$, with
$n\geq n_0$, we write
\begin{align*}
  V_n
  &:=V_{\mathbf{x}}
    =\operatorname{span}\{k(\cdot,x_i),\,1\leq i\leq n\},\\
  \Pi_n&:=\Pi_{\mathbf{x}},
  \qquad
  \Pi_n^\perp:=\Pi_{\mathbf{x}}^\perp=I-\Pi_n.
\end{align*}
The Gram matrix in
$\Hk$ of the evaluation representers
$k(\cdot,x_1),\ldots,k(\cdot,x_n)$ is
$$
  K_n
  :=
  \bigl[
    \langle k(\cdot,x_i),k(\cdot,x_j)\rangle_{\Hk}
  \bigr]_{i,j=1}^n
  =
  [k(x_i,x_j)]_{i,j=1}^n.
$$
We also write
$$
  \mathbf k_n(x):=(k(x_1,x),\ldots,k(x_n,x))^\top,\qquad
  \mathbf z_n:=(z_1,\ldots,z_n)^\top.
$$
Define the evaluation map
$$
  \mathcal E_n:\Hk\longrightarrow\mathbb R^n,
  \qquad
  \mathcal E_n f:=(f(x_1),\ldots,f(x_n))^\top.
$$
Then
$K_n=\mathcal E_n\mathcal E_n^*$ and hence
$\operatorname{Range}(K_n)=\operatorname{Range}(\mathcal E_n)$.
Thus $\mathbf z_n=\mathcal E_n f$ and
$\mathbf k_n(x)=\mathcal E_n k(\cdot,x)$ belong to
$\operatorname{Range}(K_n)$.  Writing $K_n^\dagger$ for the Moore--Penrose
inverse of $K_n$, conditioning $Z$ on $Z(x_i)=z_i$, $1\leq i\leq n$, gives
the posterior mean $\mu_n$ and \emph{normalized} posterior standard
deviation~$s_n$:
\begin{equation}
\begin{aligned}
  \mu_n(x)
  &=\mathbf k_n(x)^\top K_n^\dagger\mathbf z_n
    =(\Pi_nf)(x),\\
  s_n^2(x)
  &=k(x,x)-\mathbf k_n(x)^\top K_n^\dagger\mathbf k_n(x)
    =\lVert \Pi_n^\perp k(\cdot,x)\rVert_{\Hk}^2.
\end{aligned}
\label{eq:posterior-formulas}
\end{equation}
The formulas in \eqref{eq:posterior-formulas} remain valid for repeated
query points and singular Gram matrices.  By
\eqref{eq:tuple-power-function},
\begin{equation}
  s_n(x)
  =P_{(x_1,\ldots,x_n)}(x)
  =
  \dist_{\Hk}\bigl(k(\cdot,x),V_n\bigr).
  \label{eq:power-function}
\end{equation}
For $x\in\X$, the deterministic interpolation error satisfies
$$
  f(x)-\mu_n(x)
  =\left\langle f-\Pi_nf,
    \iota_{(x_1,\ldots,x_n)}(x)\right\rangle_{\Hk}.
$$
The posterior predictive standard deviation is $\sigma s_n(x)$.  At the next
query point, we write
\begin{equation}
\begin{aligned}
  \iota_{n+1}
  &:=\iota_{(x_1,\ldots,x_n)}(x_{n+1}),\\
  v_{n+1}
  &:=\lVert\iota_{n+1}\rVert_{\Hk}
    =P_{(x_1,\ldots,x_n)}(x_{n+1})
    =s_n(x_{n+1}).
\end{aligned}
\label{eq:selected-point-innovation}
\end{equation}
The Lo\`eve isometry
$k(\cdot,x)\mapsto Z(x)/\sigma$ also gives
$$
  v_{n+1}
  =
  \frac{1}{\sigma}
  \left\lVert
    Z(x_{n+1})
    -\mathsf E\left[
      Z(x_{n+1})\mid Z(x_1),\ldots,Z(x_n)
    \right]
  \right\rVert_{L^2}.
$$

\subsection{Sequential separation radii}

Define the sequential separation radius of $k$ at rank $m\geq1$ by
\begin{equation}
  \delta_m(k,\X)
  :=\sup_{x_1,\ldots,x_m\in\X}
    \min_{1\leq j\leq m}P_{(x_1,\ldots,x_{j-1})}(x_j).
  \label{eq:delta}
\end{equation}
Thus $\delta_m(k,\X)$ is the supremum, over all $m$-point
sequences, of their smallest successive innovation norm.

The sequence $(\delta_m(k,\X))_{m\geq1}$ is nonincreasing.  Indeed, the
minimum for an $(m+1)$-point sequence is no larger than the minimum
for its $m$-point prefix, and taking the supremum gives
$\delta_{m+1}(k,\X)\leq\delta_m(k,\X)$.  In particular,
$$
  0\leq\delta_m(k,\X)
  \leq\delta_1(k,\X)
  =\sup_{x\in\X}\sqrt{k(x,x)}.
$$
By \eqref{eq:normalization}, $\delta_m(k,\X)\leq1$.  Rescaling the kernel by
$c^2$ gives
$$
  \delta_m(c^2k,\X)=c\,\delta_m(k,\X),\qquad c>0.
$$

\begin{lemma}[Decay of the sequential separation radii]
\label{lem:delta-vanishes}
For the continuous kernel $k$ on the compact set $\X$,
\begin{equation}
  \delta_m(k,\X)\downarrow0
  \qquad(m\to\infty).
  \label{eq:delta-vanishes}
\end{equation}
\end{lemma}

\begin{proof}
The map $x\mapsto k(\cdot,x)$ is continuous, so its image is compact in
$\Hk$.  If all successive innovation norms of
$(x_1,\ldots,x_m)\in\X^m$ exceed $\varepsilon$, then, for $i<j$,
$$
  \bigl\|k(\cdot,x_j)-k(\cdot,x_i)\bigr\|_{\Hk}
  \geq P_{(x_1,\ldots,x_{j-1})}(x_j)
  >\varepsilon.
$$
Thus the $m$ evaluation representers are pairwise
$\varepsilon$-separated.  The standard packing argument for compact metric
spaces bounds their number in terms of $\varepsilon$.  Hence
$\delta_m(k,\X)\leq\varepsilon$ for all sufficiently large~$m$.
Since $\varepsilon>0$ is arbitrary, \eqref{eq:delta-vanishes} follows.
\end{proof}

\subsection{Restart after the initial design}

After the initial design, let
$\Psi_0(x):=\Pi_{n_0}^\perp k(\cdot,x)$ be the projected canonical feature
map.  Define the residual kernel $k_0$ by
\begin{equation}
  k_0(x,y):=\langle\Psi_0(x),\Psi_0(y)\rangle_{\Hk}.
  \label{eq:posterior-residual-kernel}
\end{equation}
For $g\in V_{n_0}^{\perp}$ and $x\in\X$,
$$
  \langle g,\Psi_0(x)\rangle_{\Hk}
  =\langle g,k(\cdot,x)\rangle_{\Hk}
  =g(x).
$$
The closed subspace $V_{n_0}^{\perp}$, with the norm inherited from~$\Hk$,
is therefore an RKHS with reproducing kernel $k_0$.  By uniqueness,
$\mathcal H_{k_0}=V_{n_0}^{\perp}$ isometrically.  In particular,
$k_0(\cdot,x)=\Psi_0(x)$ as functions on $\X$.

After conditioning on the initial observations, $Z$ has covariance function
$\sigma^2k_0$.  The kernel $k_0$ is continuous and positive semidefinite,
with $k_0(x,x)=s_{n_0}^2(x)\leq1$.  We use
$P_{\mathbf{x}}^{k_0}$ and $\delta_m(k_0,\X)$ for the innovation norm and
sequential separation radius obtained from~\eqref{eq:tuple-power-function}
and~\eqref{eq:delta} with $k_0$ in place of $k$.  The superscript
distinguishes this power function from $P_{\mathbf{x}}$, which is computed
with $k$.
The diagonal bound gives
$\delta_m(k_0,\X)\leq1$.  For every $n\geq n_0$, orthogonal decomposition
relative to $V_{n_0}$ gives
\begin{equation}
  s_n(x)
  =
  \dist_{\Hk}\left(
    \Psi_0(x),
    \operatorname{span}\{\Psi_0(x_{n_0+1}),\ldots,\Psi_0(x_n)\}
  \right),
  \label{eq:posterior-restart-power}
\end{equation}
where the span is $\{0\}$ when $n=n_0$.
Thus $s_n$ is the power function for $k_0$ based on the post-initial query
points $x_{n_0+1},\ldots,x_n$.

Since $\mu_{n_0}=\Pi_{n_0}f$, the initial residual
$f-\mu_{n_0}=\Pi_{n_0}^{\perp}f$ belongs to
$\mathcal H_{k_0}=V_{n_0}^{\perp}$, and
$\lVert\mu_{n_0}\rVert_{\Hk}\leq\lVert f\rVert_{\Hk}\leq B$.
The residual RKHS radius after the initial interpolation is therefore
\begin{equation}
  B_0:=
  \left(B^2-\lVert\mu_{n_0}\rVert_{\Hk}^2\right)^{1/2}.
  \label{eq:residual-rkhs-radius}
\end{equation}
The minimum-norm interpolation formula gives
$$
  \lVert\mu_{n_0}\rVert_{\Hk}^2
  =\mathbf z_{n_0}^\top K_{n_0}^\dagger\mathbf z_{n_0},
$$
so $B_0$ is computable even when $K_{n_0}$ is singular.
Since $V_{n_0}\subseteq V_n$, the projection norms satisfy
$\lVert\mu_n\rVert_{\Hk}\geq\lVert\mu_{n_0}\rVert_{\Hk}$.  Hence, for every
$n\geq n_0$,
\begin{equation}
  \lVert f-\mu_n\rVert_{\Hk}^2
  =\lVert f\rVert_{\Hk}^2-\lVert\mu_n\rVert_{\Hk}^2
  \leq B_0^2.
  \label{eq:residual-radius-error}
\end{equation}

\subsection{Expected improvement and simple regret}
\label{sec:ei-and-regret}

Because $k$ is continuous,
$x\mapsto k(\cdot,x)$ is continuous in $\Hk$, hence every
$f\in\Hk$ is continuous.  Compactness therefore ensures that a minimizer
$x^\star$ exists.  Let
$f^\star:=f(x^\star)=\min_{x\in\X}f(x)$.  Define the best observed value
$m_n$ and the simple regret $r_n$ for $n\geq n_0$ by
\begin{equation}
  m_n:=\min_{1\leq i\leq n}z_i,
  \qquad
  r_n:=m_n-f^\star.
  \label{eq:simple-regret}
\end{equation}
Since $f^\star\leq m_{n+1}\leq m_n$, simple regret is nonincreasing:
\begin{equation}
  0\leq r_{n+1}\leq r_n,
  \qquad n\geq n_0.
  \label{eq:regret-monotonicity}
\end{equation}
Define, for $n\geq n_0$,
\begin{equation}
  I_n:\X\longrightarrow[0,\infty),
  \qquad
  I_n(x):=(m_n-f(x))_+, \qquad u_+:=\max\{u,0\}.
  \label{eq:realized-improvement}
\end{equation}
This is the improvement that would be realized by querying the candidate
point $x$.
At a global minimizer, $I_n$ equals the current simple regret.  At the next
selected point, it equals the decrease in simple regret:
\begin{equation}
  I_n(x^\star)=r_n,
  \qquad
  I_n(x_{n+1})
  =(m_n-z_{n+1})_+
  =m_n-m_{n+1}
  =r_n-r_{n+1}.
  \label{eq:improvement-regret}
\end{equation}

Conditioning the Gaussian-process model~\eqref{eq:fixed-prior-model} on the
observed history
$\mathcal D_n:=((x_i,z_i))_{i=1}^n$ gives the Gaussian predictive distribution
\begin{equation}
  Z(x)\mid\mathcal D_n
  \sim\mathcal N\bigl(\mu_n(x),\sigma^2s_n^2(x)\bigr).
  \label{eq:posterior-predictive-law}
\end{equation}
Let $\mathsf E_n[\cdot]:=\mathsf E[\cdot\mid\mathcal D_n]$ denote posterior
expectation.  Following \citet[Equation~(14)]{jones1998}, the EI sampling
criterion for minimization after $n$ evaluations is
\begin{equation}
  \EI_n(x):=\mathsf E_n\bigl[(m_n-Z(x))_+\bigr].
  \label{eq:ei-definition}
\end{equation}

Let $\Phi$ and $\phi$ denote the standard normal distribution function and
density, and put
$\tau(z):=z\Phi(z)+\phi(z)$.  Since $\tau'(z)=\Phi(z)$ and
$\tau''(z)=\phi(z)>0$, the function $\tau$ is increasing and strictly convex.
Using \eqref{eq:posterior-predictive-law}, the closed form is
\begin{equation}
  \EI_n(x)
  =
  \begin{cases}
    \sigma s_n(x)\,
    \tau\left(\dfrac{m_n-\mu_n(x)}{\sigma s_n(x)}\right),
      &s_n(x)>0,\\[2mm]
    (m_n-\mu_n(x))_+,&s_n(x)=0.
  \end{cases}
  \label{eq:ei-closed-form}
\end{equation}

Continuity of $x\mapsto k(\cdot,x)$ implies that $\mu_n$ and $s_n$ are
continuous.  The expression in~\eqref{eq:ei-closed-form} is continuous in
the pair
$(m_n-\mu_n(x),s_n(x))$, including at points where $s_n(x)=0$.
Consequently, $\EI_n$ is continuous on $\X$, and compactness ensures that it
attains its maximum.

The selected point need not maximize EI exactly.  We assume instead that
there is an $\eta\in(0,1]$, independent of $n$, such that
\begin{equation}
  \EI_n(x_{n+1})
  \geq \eta\sup_{x\in\X}\EI_n(x),
  \qquad n\geq n_0.
  \label{eq:approx-ei}
\end{equation}
We call~\eqref{eq:approx-ei} the weak-EI condition.

\subsection{Consistency and density}

We call an EI trajectory consistent if $r_n\to0$.  Its query
points are dense in $\X$ if
$\overline{\{x_n,\,n\geq1\}}=\X$.  Simple-regret consistency does not
imply that the query points are dense.

The consistency and density results discussed in this subsection are
recorded for completeness and insight.  They are not used in the
quantitative bounds that follow.

The compactness proof in Appendix~\ref{app:consistency-density-proofs}
adapts the argument of \citet[Lemma~12]{vazquezBect2010} to establish
$$
  \liminf_{n\to\infty}\EI_n(x_{n+1})=0.
$$
The lower EI bound in Lemma~\ref{lem:ei-comparison} at a minimizer, the
selection condition \eqref{eq:approx-ei}, and monotonicity of simple regret
then imply $r_n\to0$.
Corollary~\ref{cor:ei-consistency} in
Subsection~\ref{sec:finite-budget-consistency} gives a second proof by applying
Lemma~\ref{lem:delta-vanishes} to the continuous residual kernel $k_0$ and
then using Theorem~\ref{thm:order-statistic-transfer}.

\begin{definition}[No-empty-ball property]
\label{def:neb}
Following \citet[Definition~3]{vazquezBect2010}, the kernel $k$ has the
no-empty-ball (NEB) property on $\X$ if, for every sequence
$(x_i)_{i\geq1}\subset\X$ and every $x\in\X$,
\begin{equation}
  P_{(x_1,\ldots,x_n)}(x)\longrightarrow0
  \quad\Longleftrightarrow\quad
  x\in\overline{\{x_i,\,i\geq1\}}.
  \label{eq:neb}
\end{equation}
Continuity of $k$ implies
$P_{(x_1,\ldots,x_n)}(x)\to0$ whenever
$x\in\overline{\{x_i,\,i\geq1\}}$.  The NEB property also requires the converse.
\end{definition}

The Mat\'ern kernels considered below have the NEB property on every compact
$\X$, whereas the squared-exponential kernel fails it when $\X$ has nonempty
interior \citep{vazquezBectKriging2010}.  In
dimension one, \citet[Theorem~1]{yarotsky2013} proves this failure for
stationary kernels whose spectral densities decay at least exponentially.

When $k$ has the NEB property,
\citet[Theorem~6 and Remark~8]{vazquezBect2010} proved that every exact~EI
policy generates a dense sequence of query points in $\X$ for every
$f\in\Hk$.  Theorem~\ref{thm:neb-density} extends this result to weak-EI
policies.

\begin{theorem}[Density under the NEB property]
\label{thm:neb-density}
If $k$ has the NEB property, then every fixed-prior EI
trajectory satisfying \eqref{eq:approx-ei} is dense in $\X$.
\end{theorem}

\begin{proof}
See Appendix~\ref{app:consistency-density-proofs}.
\end{proof}

\section{EI comparison and one-step regret bound}
\label{sec:ei-comparison}

\subsection{Scalar EI comparison}
\label{sec:scalar-ei-comparison}

The Cauchy--Schwarz inequality, the power-function identity
\eqref{eq:power-function}, and the residual-norm bound
\eqref{eq:residual-radius-error} give, for every $n\geq n_0$ and $x\in\X$,
\begin{equation}
  |f(x)-\mu_n(x)|\leq B_0 s_n(x).
  \label{eq:rkhs-error}
\end{equation}
Set $a:=B_0/\sigma$.  The scalar comparison below depends on $B_0$ and
$\sigma$ only through this dimensionless ratio.
Define, for $n\geq n_0$ and a candidate point $x$ with $s_n(x)>0$, the
standardized gaps
\begin{equation}
  u:=\frac{m_n-\mu_n(x)}{\sigma s_n(x)},
  \qquad
  t:=\frac{m_n-f(x)}{\sigma s_n(x)}.
  \label{eq:standardized-gaps}
\end{equation}
By \eqref{eq:rkhs-error}, $|u-t|\leq a$.  Moreover,
$I_n(x)=\sigma s_n(x)\,t_+$ and
\begin{equation}
  \EI_n(x)=\sigma s_n(x)\,\tau(u).
  \label{eq:ei-tau}
\end{equation}

We now seek affine bounds of the form
$$
  \kappa I_n(x)
  \leq \EI_n(x)
  \leq I_n(x)+D\,\sigma s_n(x),
  \qquad \kappa,D\geq0.
$$
It is therefore enough to prove
\begin{equation}
  \kappa t_+
  \leq \tau(u)
  \leq t_++D
  \qquad
  \text{for every }u,t\in\mathbb R
  \text{ with }|u-t|\leq a.
  \label{eq:scalar-ei-comparison}
\end{equation}
For $a>0$, we construct the coefficient $\kappa$ in the lower bound
$\kappa t_+\leq\tau(u)$ from the tangent to $\tau$ through $(-a,0)$.
A tangent at abscissa $z$ passes through $(-a,0)$ when
$$
  \tau(z)=(z+a)\tau'(z),
$$
or equivalently $\phi(z)=a\Phi(z)$.  With $h:=\phi/\Phi$, this equation becomes
$h(z)=a$.  The ratio~$h$ satisfies
$$
  h'(z)=-h(z)\{z+h(z)\}<0,
  \qquad
  z+h(z)=\frac{\tau(z)}{\Phi(z)}>0,
$$
and, as $z$ increases from $-\infty$ to $+\infty$, $h(z)$ decreases
strictly from $+\infty$ to $0$.

\begin{lemma}[Upper and lower bounds for EI]
\label{lem:ei-comparison}
Let $z_a$ be the unique solution of $h(z)=a$ when $a>0$.  Define, for
$a\geq0$,
\begin{equation}
  \kappa_a:=
  \begin{cases}
    1, & a=0,\\
    \Phi(z_a), & a>0.
  \end{cases}
  \label{eq:ei-comparison-constant}
\end{equation}
For every $n\geq n_0$ and $x\in\X$,
\begin{equation}
  \max\bigl\{I_n(x)-B_0s_n(x),\ \kappa_a I_n(x)\bigr\}
  \leq \EI_n(x)
  \leq I_n(x)+\tau(a)\,\sigma s_n(x).
  \label{eq:ei-comparison}
\end{equation}
\end{lemma}

\begin{proof}
Fix $n\geq n_0$ and $x\in\X$, and write $s=s_n(x)$.  Suppose first that
$s>0$, and use the standardized gaps $u,t$ from
\eqref{eq:standardized-gaps}.
For the upper bound, if $t\leq0$, then
$I_n(x)=0$ and $u\leq a$, so
$\EI_n(x)\leq\tau(a)\,\sigma s$.
If $t>0$, then
$$
  \frac{d}{dt}\{\tau(t+a)-t\}
  =\Phi(t+a)-1\leq0.
$$
Together with $u\leq t+a$ and monotonicity of $\tau$, this gives
$$
  \EI_n(x)
  \leq\sigma s\,\tau(t+a)
  \leq I_n(x)+\tau(a)\,\sigma s.
$$

When $I_n(x)=0$, the two lower bounds follow from nonnegativity of EI.
Suppose now that $I_n(x)>0$, so $t>0$.  The identity
$\tau(z)=z+\tau(-z)$, the nonnegativity of $\tau$, and the inequality
$u\geq t-a$ give
$$
  \EI_n(x)\geq I_n(x)-B_0s.
$$
For the multiplicative lower bound, if $a=0$, then $u=t$ and
$\tau(u)=\tau(t)\geq t$.  If~$a>0$, convexity of $\tau$, whose
derivative is $\Phi$, and tangency at $z_a$ give
$$
  \tau(t-a)
  \geq\tau(z_a)+\Phi(z_a)(t-a-z_a)
  =\Phi(z_a)t=\kappa_a t.
$$
Since $u\geq t-a$ and $\tau$ is nondecreasing, this proves
$\EI_n(x)\geq\kappa_a I_n(x)$.
If $s=0$, \eqref{eq:rkhs-error} gives $\mu_n(x)=f(x)$, and the
degenerate predictive distribution gives $\EI_n(x)=I_n(x)$.  This proves
\eqref{eq:ei-comparison}.
\end{proof}

The constants $\kappa_a$ and $\tau(a)$ are optimal in
\eqref{eq:scalar-ei-comparison}.  For $a>0$,
$$
  h(-a)-a=\frac{\tau(-a)}{\Phi(-a)}>0.
$$
Since $h$ is decreasing and $h(z_a)=a$, this gives $z_a>-a$.  Hence
$a+z_a>0$.  The pair
$(u,t)=(z_a,a+z_a)$ satisfies $|u-t|=a$ and attains equality in the
multiplicative lower bound.  When $a=0$, the pairs
$(u,t)=(t,t)$ with $t>0$ satisfy the constraint in
\eqref{eq:scalar-ei-comparison} and
$$
  \frac{\tau(t)}{t}
  =\Phi(t)+\frac{\phi(t)}{t}
  \longrightarrow1
  \qquad\text{as }t\to\infty.
$$
Thus the lower inequality in \eqref{eq:scalar-ei-comparison} cannot hold with
$\kappa>1$.  The pair
$(u,t)=(a,0)$ satisfies $|u-t|=a$ and attains equality in the upper bound.
Therefore $\kappa_a$ is the largest value of $\kappa$, and $\tau(a)$ is the
smallest value of $D$, for which \eqref{eq:scalar-ei-comparison} holds.

The map
$a\mapsto\kappa_a$ is nonincreasing, whereas $a\mapsto\tau(a)$ is
increasing.  Moreover, as $a\to\infty$, $h(z_a)=a$ implies
$z_a\to-\infty$, and hence $\kappa_a=\Phi(z_a)\to0$.

The scalar constants above sharpen those in \citet[Lemma~8]{bull2011}, whose
model uses an unknown constant mean integrated under a flat prior.

\subsection{One-step regret bound}
\label{sec:one-step-recursion}

\begin{proposition}[One-step regret bound]
\label{prop:one-step-regret}
Define
\begin{equation}
  q:=1-\eta\kappa_a\in[0,1),
  \qquad
  C:=\sigma\tau(a)>0.
  \label{eq:recursion-constants}
\end{equation}
For every $n\geq n_0$, if the selected point $x_{n+1}$ satisfies
\eqref{eq:approx-ei}, then
\begin{equation}
    r_{n+1}
    \leq q r_n+C s_n(x_{n+1}).
  \label{eq:one-step}
\end{equation}
\end{proposition}

\begin{proof}
The multiplicative lower bound in \eqref{eq:ei-comparison} at a global
minimizer and the additive upper bound at the selected point, together with
\eqref{eq:improvement-regret} and \eqref{eq:approx-ei}, give
$$
  \eta\kappa_a r_n
  \leq \EI_n(x_{n+1})
  \leq r_n-r_{n+1}+\sigma\tau(a)s_n(x_{n+1}).
$$
Rearranging proves \eqref{eq:one-step}.
\end{proof}

For $n\geq n_0$ with $r_n>0$, the one-step regret bound
\eqref{eq:one-step} gives strict regret decrease whenever
$Cs_n(x_{n+1})<(1-q)r_n$.  Otherwise, the right-hand side of
\eqref{eq:one-step} is at least $r_n$, and
\eqref{eq:regret-monotonicity} gives $r_{n+1}\leq r_n$.  For $T>n_0$, the values
$v_{n+1}=s_n(x_{n+1})$, $n_0\leq n<T$, are the post-initial
selected-point innovation norms.  For $1\leq j\leq T-n_0$,
Lemma~\ref{lem:innovation-order-statistics} bounds their $j$th largest value
by $\delta_j(k_0,\X)$.

\begin{remark}[Uniform constants]
\label{rem:ei-constants}
The initial values $f(x_1),\ldots,f(x_{n_0})$ enter the constants $q$ and
$C$ through $B_0$ in
\eqref{eq:residual-rkhs-radius}.
Define
\begin{equation}
  q_B:=1-\eta\kappa_{B/\sigma},
  \qquad
  C_B:=\sigma\tau(B/\sigma).
  \label{eq:uniform-recursion-constants}
\end{equation}
Since $a=B_0/\sigma\leq B/\sigma$, the monotonicity of
$a\mapsto\kappa_a$ and $a\mapsto\tau(a)$ gives
$q\leq q_B<1$ and $C\leq C_B$.
Since $r_n$ and $s_n(x_{n+1})$ are nonnegative,
$$
  q r_n+C s_n(x_{n+1})
  \leq q_Br_n+C_Bs_n(x_{n+1}).
$$
Thus \eqref{eq:one-step} remains valid when $q,C$ are replaced by $q_B,C_B$.
The constants $q_B$ and $C_B$ depend only on $B$, $\sigma$, and $\eta$.
\end{remark}

(Appendix~\ref{app:optimizer-uncertainty} uses both lower bounds in
\eqref{eq:ei-comparison} to derive the refined one-step regret bound
\eqref{eq:optimizer-uncertainty-recursion}, which involves both
$s_n(x_{n+1})$ and $s_n(x^\star)$ for a minimizer $x^\star$.)

\section{Ranked innovation norms and regret bounds}
\label{sec:innovation-transfer}

The separation radii of $k$ bound those of $k_0$, which in turn bound the
ranked selected-point innovation norms.  The one-step regret bound
\eqref{eq:one-step} then yields finite-budget bounds and simple-regret
consistency.

\subsection{Separation radii after the initial design}

By \eqref{eq:posterior-restart-power}, the normalized posterior standard
deviations at the post-initial selected points are successive innovation norms
for $k_0$.  Sections~\ref{sec:width-geometry} and~\ref{sec:kernels} give
quantitative bounds for the separation radii of $k$.

\begin{lemma}[Residual-kernel separation radii]
\label{lem:residual-envelope-domination}
For every $m\geq1$,
\begin{equation}
  \delta_m(k_0,\X)\leq\delta_m(k,\X).
  \label{eq:residual-envelope-domination}
\end{equation}
\end{lemma}

\begin{proof}
Recall from \eqref{eq:posterior-residual-kernel} that
$$
  k_0(x,y)=\langle\Psi_0(x),\Psi_0(y)\rangle_{\Hk}.
$$
Write, for an ordered tuple $\mathbf{x}=(x_1,\ldots,x_m)$,
$\mathbf{x}_0:=\varnothing$ and
$\mathbf{x}_j:=(x_1,\ldots,x_j)$ for $1\leq j\leq m$.
Since $k_0(\cdot,x)=\Psi_0(x)$ and $\mathcal H_{k_0}$ carries the norm
inherited from $\Hk$, for $1\leq j\leq m$,
$$
  P^{k_0}_{\mathbf{x}_{j-1}}(x_j)
  =\dist_{\Hk}\left(
    \Psi_0(x_j),
    \operatorname{span}\{\Psi_0(x_i),\,i<j\}
  \right).
$$
Thus the distances defining $\delta_m(k_0,\X)$ may be computed in~$\Hk$.
Since $\Psi_0(x)=\Pi_{n_0}^\perp k(\cdot,x)$,
$$
  \operatorname{span}\{\Psi_0(x_i),\,i<j\}
  =\Pi_{n_0}^\perp V_{\mathbf{x}_{j-1}}.
$$
The projector $\Pi_{n_0}^\perp$ is a contraction, so
$$
  \dist_{\Hk}\left(
    \Pi_{n_0}^\perp k(\cdot,x_j),
    \Pi_{n_0}^\perp V_{\mathbf{x}_{j-1}}
  \right)
  \leq
  \dist_{\Hk}\left(
    k(\cdot,x_j),
    V_{\mathbf{x}_{j-1}}
  \right).
$$
Taking the minimum over $j$ and then the supremum over tuples proves
\eqref{eq:residual-envelope-domination}.
\end{proof}

Thus bounds for $\delta_m(k,\X)$ also apply to $\delta_m(k_0,\X)$.  The
inequality can be strict because $\Pi_{n_0}^\perp$ may shorten these
distances by removing components in $V_{n_0}$, the span generated by the
initial design.

\subsection{Ranked selected-point innovation norms}

The following lemma bounds the $j$th largest normalized posterior standard
deviation among the post-initial selected points.

\begin{lemma}[Bounds for the ranked selected-point innovation norms]
\label{lem:innovation-order-statistics}
Let $T>n_0$ and write $N:=T-n_0$.  Set, for $1\leq i\leq N$,
\begin{equation}
  \widetilde v_i
  :=v_{n_0+i}
  =s_{n_0+i-1}(x_{n_0+i}).
  \label{eq:post-initial-innovation-reindexing}
\end{equation}
Let $v_{[1]}\geq\cdots\geq v_{[N]}$ be the nonincreasing rearrangement of
$\widetilde v_1,\ldots,\widetilde v_N$.  Then
\begin{equation}
  v_{[j]}\leq\delta_j(k_0,\X)
  \leq\delta_j(k,\X),\qquad 1\leq j\leq N.
  \label{eq:innovation-order-statistics}
\end{equation}
\end{lemma}

\begin{proof}[Proof idea]
Fix $j$ and list chronologically the indices of the $j$ largest
post-initial innovation norms.  At each of these indices, the span generated by
all earlier post-initial query points contains the span generated by the
earlier points in this $j$-point subsequence.  Each selected innovation norm
is therefore at most the corresponding successive innovation norm of the
subsequence.  The smallest selected value is $v_{[j]}$.  Taking minima and
using the definition of $\delta_j$ gives
$v_{[j]}\leq\delta_j(k_0,\X)$.
Lemma~\ref{lem:residual-envelope-domination} gives the second inequality in
\eqref{eq:innovation-order-statistics}.
Appendix~\ref{app:ranked-innovation-proof} gives the complete proof.
\end{proof}

In particular, for $1\leq m\leq N$ and $\varepsilon\geq0$, if
$\delta_m(k_0,\X)\leq\varepsilon$, then at most $m-1$ of the post-initial
values $v_{n_0+1},\ldots,v_T$ can exceed $\varepsilon$.

\subsection{Finite-budget bounds}

Over a finite horizon, the positions of the smallest post-initial innovation
norms are known only retrospectively.  The following lemma expresses the
recursion as a minimum over subsets of indices.

Define, for $q\in[0,1)$, $C>0$, and $v\geq0$, the map
$F_v:[0,\infty)\to[0,\infty)$ by
\begin{equation}
  F_v(u):=\min\{u,qu+Cv\}.
  \label{eq:capped-recursion-map}
\end{equation}

\begin{lemma}[Recursion along a subsequence]
\label{lem:capped-certificate}
Let $N\geq1$, and let $u_0,\ldots,u_N$ and $v_1,\ldots,v_N$ be
nonnegative numbers such that
$$
  u_i\leq u_{i-1},
  \qquad
  u_i\leq q u_{i-1}+Cv_i,
  \qquad 1\leq i\leq N,
$$
where $q\in[0,1)$ and $C>0$.  Given $U_0\geq u_0$, define
\begin{equation}
  U_i:=\min\{U_{i-1},qU_{i-1}+Cv_i\},
  \qquad 1\leq i\leq N.
  \label{eq:pathwise-capped-certificate}
\end{equation}
Then $u_i\leq U_i$ for every $i$.  Write each nonempty subset
$S\subseteq\{1,\ldots,N\}$ as
$S=\{i_1<\cdots<i_{|S|}\}$.  Then
\begin{equation}
  U_N
  =
  \min_{S\subseteq\{1,\ldots,N\}}
  \biggl\{
    q^{|S|} U_0
    +C\sum_{h=1}^{|S|}q^{|S|-h}v_{i_h}
  \biggr\},
  \label{eq:pathwise-subset-certificate}
\end{equation}
where the expression in braces is interpreted as $U_0$ when
$S=\varnothing$.
\end{lemma}

\begin{proof}[Proof idea]
The map $F_v$ is nondecreasing, so induction gives $u_i\leq U_i$.
At each step, \eqref{eq:pathwise-capped-certificate} takes either the first
or the second branch of the minimum.  The indices at which it takes the
second branch form a subset $S$, and expanding those affine steps gives the
expression in braces in \eqref{eq:pathwise-subset-certificate}.  Induction on
$i$ gives the minimum over all subsets.
Appendix~\ref{app:capped-certificate-proof} gives the complete proof.
\end{proof}

\begin{theorem}[Finite-budget bound]
\label{thm:order-statistic-transfer}
Let $(r_n)_{n\geq n_0}$ be any nonnegative nonincreasing sequence
satisfying, for fixed $q\in[0,1)$ and $C>0$,
\begin{equation}
  r_{n+1}\leq q r_n+C v_{n+1},
  \qquad n\geq n_0.
  \label{eq:finite-budget-recursion}
\end{equation}
Let $T>n_0$, put $N:=T-n_0$, and let
$v_{[1]}\geq\cdots\geq v_{[N]}$ be the decreasing rearrangement of the
post-initial selected-point innovation norms
$v_{n_0+1},\ldots,v_T$.  Let
$(\delta_m(k_0,\X))_{m\geq1}$ be the sequential separation radii of the
residual kernel $k_0$ on $\X$.  Set $M:=N-L+1$ for $1\leq L\leq N$.  Then
\begin{equation}
  \begin{aligned}
  r_T
  &\leq q^Lr_{n_0}
    +C\sum_{j=0}^{L-1}q^jv_{[M+j]}\\
  &\leq q^Lr_{n_0}
    +C\sum_{j=0}^{L-1}q^j\delta_{M+j}(k_0,\X).
  \end{aligned}
  \label{eq:order-transfer-refined}
\end{equation}
Consequently,
\begin{equation}
  r_T
  \leq q^Lr_{n_0}
  +C\frac{1-q^L}{1-q}\delta_M(k_0,\X).
  \label{eq:order-transfer-coarse}
\end{equation}
\end{theorem}

\begin{proof}[Proof idea]
Apply the subset formula in Lemma~\ref{lem:capped-certificate} to the
chronological positions of the $L$ smallest post-initial innovation norms.
At every other position, use $r_{n+1}\leq r_n$ instead of
\eqref{eq:finite-budget-recursion}, so only these $L$ innovation norms appear
in the iterated bound.  The rearrangement
inequality bounds the weighted selected values by
$\sum_{j=0}^{L-1}q^jv_{[M+j]}$.  Lemma~\ref{lem:innovation-order-statistics}
then replaces each ranked innovation norm by
$\delta_{M+j}(k_0,\X)$, giving \eqref{eq:order-transfer-refined}.
Monotonicity of the separation radii gives
\eqref{eq:order-transfer-coarse}.
Appendix~\ref{app:ranked-transfer-proof} gives the complete proof.
\end{proof}

\begin{remark}
When $q=0$, taking $L=1$ in \eqref{eq:order-transfer-coarse} and using
regret monotonicity gives
$$
  r_T
  \leq
  \min\{r_{n_0},C\delta_{T-n_0}(k_0,\X)\}.
$$
\end{remark}

\paragraph{Application to fixed-prior EI.}
For fixed-prior EI, \eqref{eq:regret-monotonicity},
Proposition~\ref{prop:one-step-regret}, and
\eqref{eq:selected-point-innovation} verify the hypotheses of
Theorem~\ref{thm:order-statistic-transfer}.  For every $f\in\Hk$ with
$\lVert f\rVert_{\Hk}\leq B$, Remark~\ref{rem:ei-constants} shows that
\eqref{eq:one-step} remains valid with $q_B,C_B$.  The theorem therefore
applies with these constants to every fixed-prior EI trajectory
satisfying~\eqref{eq:approx-ei}.

In \eqref{eq:order-transfer-coarse}, increasing $L$ makes the contraction
term $q^Lr_{n_0}$ nonincreasing, but decreases $M=N-L+1$ and makes
$\delta_M(k_0,\X)$ nondecreasing.
Section~\ref{sec:kernels} chooses $L$ for the polynomial and exponential
bounds on the separation radii in
\eqref{eq:polynomial-envelope} and~\eqref{eq:log-enhanced-envelope}.

\subsection{Consistency from the finite-budget bound}
\label{sec:finite-budget-consistency}

\begin{corollary}[Consistency from decay of the separation radii]
\label{cor:qualitative-consistency}
Suppose that the separation radii $\delta_m(k_0,\X)$ tend to zero.  If $(r_n)$
satisfies the hypotheses of
Theorem~\ref{thm:order-statistic-transfer}, then
$$
  r_T\longrightarrow0
  \qquad(T\to\infty).
$$
\end{corollary}

\begin{proof}
Let $T\geq n_0+2$, put $N:=T-n_0$, take $L=\lfloor N/2\rfloor$, and set
$M:=N-L+1$.  As $T\to\infty$, both $L$ and $M$ tend to infinity, and
$q^L\to0$ because $0\leq q<1$.  The assumed decay of the separation radii and
\eqref{eq:order-transfer-coarse} give
$$
  r_T
  \leq q^Lr_{n_0}
  +\frac{C}{1-q}\delta_M(k_0,\X)
  \longrightarrow0.
$$
\end{proof}

\begin{corollary}[Consistency of EI]
\label{cor:ei-consistency}
For every fixed-prior EI trajectory satisfying~\eqref{eq:approx-ei},
$$
  r_n\longrightarrow0.
$$
\end{corollary}

\begin{proof}
By \eqref{eq:regret-monotonicity},
Proposition~\ref{prop:one-step-regret}, and
\eqref{eq:selected-point-innovation}, the simple-regret sequence and
selected-point innovation norms satisfy the hypotheses of
Theorem~\ref{thm:order-statistic-transfer}.  The residual kernel $k_0$ is
continuous, so Lemma~\ref{lem:delta-vanishes} gives
$\delta_m(k_0,\X)\to0$.  Corollary~\ref{cor:qualitative-consistency}
therefore gives $r_n\to0$.
\end{proof}

For quantitative bounds,
Lemma~\ref{lem:residual-envelope-domination} shows that it suffices to bound
$\delta_m(k,\X)$.  Section~\ref{sec:width-geometry} obtains such bounds from
Gram determinants and Kolmogorov widths.

\section{From Kolmogorov widths to sequential separation radii}
\label{sec:width-geometry}

\subsection{From innovation norms to Gram determinants}

The determinant and width bounds in this section do not use the Euclidean
structure of $\X$.
We therefore let $\mathcal D$ be a nonempty set and
$k:\mathcal D\times\mathcal D\to\mathbb R$ a positive-semidefinite kernel.
We use the innovation and separation-radius notation of
\eqref{eq:tuple-power-function} and~\eqref{eq:delta}, with $\mathcal D$ in
place of $\X$.

For an ordered tuple
$\mathbf{x}=(x_1,\ldots,x_m)\in\mathcal D^m$, the Gram matrix in
$\mathcal H_k$ of the evaluation representers
$k(\cdot,x_1),\ldots,k(\cdot,x_m)$ is
\begin{equation}
  K_{\mathbf{x}}
  :=
  \bigl[
    \langle k(\cdot,x_i),k(\cdot,x_j)\rangle_{\mathcal H_k}
  \bigr]_{i,j=1}^m
  =
  [k(x_i,x_j)]_{i,j=1}^m.
  \label{eq:gram-matrix}
\end{equation}

Applying Gram--Schmidt orthogonalization, without normalization, to the
ordered evaluation representers produces their successive innovations.
These vectors are orthogonal, and the Gram determinant is the product of
their squared norms.
The minimum of the innovation norms is therefore at most their geometric
mean, $\det(K_{\mathbf x})^{1/(2m)}$.
\citet[Lemma~2.1]{li2024} use the same identity to bound products of
successive distances, and \citet{santin2024} use it to bound kernel Gram
determinants and show that the polynomial order of
$\det(K_{\mathbf x})^{1/m}$ is sharp for quasi-uniform point sets under their
domain and Sobolev regularity assumptions.

\begin{lemma}[Gram determinants and innovation norms]
\label{lem:gram-volume}
Let $\mathcal D$ be nonempty and let
$k:\mathcal D\times\mathcal D\to\mathbb R$ be positive semidefinite.  For
every integer $m\geq1$, with $K_{\mathbf{x}}$ as in
\eqref{eq:gram-matrix},
\begin{equation}
  \delta_m(k,\mathcal D)
  \leq
  \sup_{\mathbf{x}\in\mathcal D^m}\det(K_{\mathbf{x}})^{1/(2m)}.
  \label{eq:gram-volume-envelope}
\end{equation}
\end{lemma}

\begin{proof}
Fix $\mathbf{x}=(x_1,\ldots,x_m)\in\mathcal D^m$.
At each $j=1,\ldots,m$, write
$$
  \iota_j:=\iota_{(x_1,\ldots,x_{j-1})}(x_j)
$$
for the innovation of $k(\cdot,x_j)$, that is, its component orthogonal to
the span of the preceding evaluation representers.  Since each $\iota_j$
equals $k(\cdot,x_j)$ minus a linear combination of those representers, the
transformation from
$$
  \bigl(k(\cdot,x_1),\ldots,k(\cdot,x_m)\bigr)
  \quad\text{to}\quad
  (\iota_1,\ldots,\iota_m)
$$
is unit triangular.  The corresponding Gram matrices are congruent through a
unit-triangular matrix and hence have the same determinant, even when the
representers are linearly dependent.  The innovations $\iota_j$ are mutually
orthogonal, and
$\lVert\iota_j\rVert_{\Hk}=P_{(x_1,\ldots,x_{j-1})}(x_j)$.  Therefore
\begin{equation}
  \det K_{\mathbf{x}}
  =\prod_{j=1}^m\lVert\iota_j\rVert_{\Hk}^2
  =\prod_{j=1}^m
    \bigl(P_{(x_1,\ldots,x_{j-1})}(x_j)\bigr)^2.
  \label{eq:gram-volume-identity}
\end{equation}
The minimum of nonnegative numbers is at most their geometric mean.  Hence
$$
  \min_{1\leq j\leq m}
    P_{(x_1,\ldots,x_{j-1})}(x_j)
  \leq\det(K_{\mathbf{x}})^{1/(2m)}.
$$
Taking the supremum over $\mathbf{x}\in\mathcal D^m$ proves
\eqref{eq:gram-volume-envelope}.
\end{proof}

\subsection{From Kolmogorov widths to Gram eigenvalues}

When $\{k(\cdot,x),\,x\in\mathcal D\}$ is compact in $\mathcal H_k$, the
packing argument gives the qualitative conclusion
$\delta_m(k,\mathcal D)\to0$.  We use
Kolmogorov widths for quantitative bounds.
The Kolmogorov widths \citep{kolmogorov1936} of the set of evaluation
representers of $k$ are
\begin{equation}
  d_m(k,\mathcal D)
  :=
  \inf_{\substack{V\subset\mathcal H_k\\ \dim V\leq m}}
  \sup_{x\in\mathcal D}
  \dist_{\mathcal H_k}\bigl(k(\cdot,x),V\bigr),
  \qquad m=0,1,2,\ldots.
  \label{eq:kernel-width}
\end{equation}
In particular,
$d_0(k,\mathcal D)=\sup_{x\in\mathcal D}\sqrt{k(x,x)}$.

For compact subsets of a Hilbert space,
\citet[Section~2]{binev2011} formulate weak greedy approximation through a
lower-triangular matrix.
\citet[Theorem~3.2]{devore2013} use this representation to give a product
inequality for successive weak-greedy approximation errors in terms of
Kolmogorov widths.
\citet[Theorem~1]{wenzelSantinHaasdonk2023} give a product inequality for the
successive innovation norms along any sequential design.
When $d_m(k,\mathcal D)=O(m^{-\alpha})$ with $\alpha>0$,
\citet[Corollary~2]{wenzelSantinHaasdonk2023} bound the geometric mean of the
successive innovation norms with indices $m+1,\ldots,2m$ by
$O(m^{-\alpha}(\log m)^\alpha)$.
In the proof of Lemma~\ref{lem:width-product}, each $i_j<j$ is the
approximation dimension used for the $j$th Gram eigenvalue.  These dimensions
may be chosen separately before the eigenvalue bounds are multiplied.

\begin{lemma}[Gram determinants and Kolmogorov widths]
\label{lem:width-product}
Let $\mathcal D$ be nonempty, let
$k:\mathcal D\times\mathcal D\to\mathbb R$ be positive semidefinite, and
suppose that $\sup_{x\in\mathcal D}k(x,x)<\infty$.  For every integer
$m\geq1$ and every choice of approximation dimensions
$i_1,\ldots,i_m$ satisfying $0\leq i_j<j$, one has
\begin{equation}
  \begin{aligned}
    \delta_m(k,\mathcal D)
    &\leq
    \sup_{\mathbf{x}\in\mathcal D^m}\det(K_{\mathbf{x}})^{1/(2m)}\\
    &\leq
    \biggl\{
      \prod_{j=1}^m
      \sqrt{\frac{m}{j-i_j}}\,d_{i_j}(k,\mathcal D)
    \biggr\}^{1/m}.
  \end{aligned}
  \label{eq:rank-profile-width-product}
  \end{equation}
In particular, $i_j=\lfloor j/2\rfloor$ gives
\begin{equation}
  \delta_m(k,\mathcal D)
  \leq
  \sup_{\mathbf{x}\in\mathcal D^m}\det(K_{\mathbf{x}})^{1/(2m)}
  \leq
  \sqrt{2e}
  \biggl\{
    \prod_{j=1}^m d_{\lfloor j/2\rfloor}(k,\mathcal D)
  \biggr\}^{1/m}.
  \label{eq:width-product}
\end{equation}
\end{lemma}

\begin{proof}
See Appendix~\ref{app:width-product-proof}.
\end{proof}

\subsection{Polynomial and exponential bounds on the widths}

\begin{corollary}[Separation-radius decay from Kolmogorov widths]
\label{cor:width-envelopes}
Suppose that $k(x,x)\leq1$ on $\mathcal D$.
\begin{enumerate}
\item If, for some $A,\alpha>0$,
\begin{equation}
  d_m(k,\mathcal D)\leq A m^{-\alpha},
  \qquad m\geq1,
  \label{eq:polynomial-width-decay}
\end{equation}
then there is $C_{\mathrm w}>0$, depending on $A,\alpha$, such that
\begin{equation}
  \delta_m(k,\mathcal D)\leq C_{\mathrm w}m^{-\alpha},
  \qquad m\geq1.
  \label{eq:polynomial-separation-decay}
\end{equation}
\item Let $0<\gamma\leq1$.  If, for some $A,b>0$,
\begin{equation}
  d_m(k,\mathcal D)
  \leq A\exp\{-b m^\gamma\log(em)\},
  \qquad m\geq1,
  \label{eq:stretched-exponential-width-decay}
\end{equation}
then, for every $0<b'<b/(1+\gamma)$, there is $A'>0$ such that
\begin{equation}
  \delta_m(k,\mathcal D)
  \leq A'\exp\{-b'm^\gamma\log(em)\},
  \qquad m\geq1.
  \label{eq:stretched-exponential-separation-decay}
\end{equation}
Here $A'$ may depend on $A,b,\gamma,b'$ but not on $m$.
\end{enumerate}
\end{corollary}

\begin{proof}
See Appendix~\ref{app:width-envelopes-proof}.
\end{proof}

We choose the dimensions $i_j$ in the widths $d_{i_j}(k,\mathcal D)$
differently in the two cases.  For polynomial widths,
$i_j=\lfloor j/2\rfloor$ keeps the geometric mean of the prefactors in
\eqref{eq:rank-profile-width-product} bounded independently of $m$ and yields
$\delta_m(k,\mathcal D)=O(m^{-\alpha})$.  For the exponential bound
\eqref{eq:stretched-exponential-separation-decay}, choosing $i_j=j-1$
gives the prefactor $\sqrt m$, and
$$
  \frac1m\sum_{i=1}^{m-1}i^\gamma\log(ei)
  \sim\frac{m^\gamma\log m}{1+\gamma},
  \qquad m\to\infty.
$$
The factor $\sqrt m$ and the $o(m^\gamma\log m)$ remainder are absorbed by
the positive difference $b/(1+\gamma)-b'$.

\section{Rates for Mat\'ern and squared-exponential kernels}
\label{sec:kernels}

\subsection{From polynomial decay of the separation radii to regret}

\begin{proposition}[Regret from polynomial decay of the separation radii]
\label{prop:polynomial-transfer}
Suppose that, for some $A>0$, $\alpha>0$, $\beta\geq0$, and all
sufficiently large~$m$,
\begin{equation}
  \delta_m(k_0,\X)
  \leq A m^{-\alpha}\{\log(em)\}^{\beta}.
  \label{eq:polynomial-envelope}
\end{equation}
Then, along every fixed-prior EI trajectory satisfying~\eqref{eq:approx-ei},
\begin{equation}
  r_T
  =O\bigl(N^{-\alpha}(\log N)^{\beta}\bigr),
  \qquad N=T-n_0\longrightarrow\infty.
  \label{eq:polynomial-rate}
\end{equation}
For fixed $f$, the implied constant and the starting index may depend on
$A,\alpha,\beta,r_{n_0},q,C$ and on the rank from which
\eqref{eq:polynomial-envelope} holds.  If $f$ ranges over the ball
$\{h\in\Hk:\lVert h\rVert_{\Hk}\leq B\}$, the implied constant and the
starting index can instead be chosen uniformly in $f$.  This uniform choice
may depend on $B,\sigma,\eta,A,\alpha,\beta,k,\X$, the initial design, and
the rank from which \eqref{eq:polynomial-envelope} holds.
\end{proposition}

\begin{proof}
See Appendix~\ref{app:polynomial-transfer-proof}.
\end{proof}

\paragraph{Proof idea.}
Apply Theorem~\ref{thm:order-statistic-transfer} with $L=1$ when $q=0$.
When $0<q<1$, choose $L=O(\log N)$ so that
$q^L=O(N^{-\alpha-1})$ and $M=N-L+1\sim N$.  Substitution in
\eqref{eq:order-transfer-coarse} gives~\eqref{eq:polynomial-rate}.

\subsection{Mat\'ern kernels}

\begin{proposition}[Mat\'ern separation-radius decay]
\label{prop:matern-separation-radius}
Let $k_\nu$ be the restriction to a nonempty compact
$\X\subset\R^d$ of a stationary isotropic Mat\'ern kernel on $\R^d$ with unit
diagonal, fixed lengthscale, and smoothness $\nu>0$.  Then
\begin{equation}
  \delta_m(k_\nu,\X)
  =O\bigl(m^{-\nu/d}\bigr).
  \label{eq:matern-innovation}
\end{equation}
\end{proposition}

\begin{proof}
See Appendix~\ref{app:matern-separation-radius}.
\end{proof}

\paragraph{Proof idea.}
For a bounded open cube $\mathcal Q$ containing $\X$, let
$k_{\nu,\mathcal Q}$ be the restriction to
$\mathcal Q\times\mathcal Q$ of the Mat\'ern kernel on $\R^d$ appearing in
the proposition.  Its spectral density is comparable at high frequencies
to $\lVert\omega\rVert^{-d-2\nu}$.
\citet[Theorem~3.2]{schabackWendland2002} therefore give a power-function
estimate of order $h^\nu$, where $h$ is the fill distance.  For every
$m\geq1$, one can choose a Cartesian grid in $\mathcal Q$ with at most $m$
points and fill distance $O(m^{-1/d})$.  Its kernel translates span a space
of dimension at most $m$.  Restricting this space from $\mathcal Q$ to $\X$
does not increase its approximation error and gives
\begin{equation}
  d_m(k_\nu,\X)
  \leq d_m(k_{\nu,\mathcal Q},\mathcal Q)
  =O(m^{-\nu/d}).
  \label{eq:matern-width-summary}
\end{equation}
Applying Corollary~\ref{cor:width-envelopes} to
\eqref{eq:matern-width-summary} gives the separation-radius bound
\eqref{eq:matern-innovation}.
Only the auxiliary cube must satisfy the interior cone condition, so no
boundary regularity of $\X$ is required.
\citet{wenzelSantinHaasdonk2025} transfer interpolation rates from a larger
domain to weak $P$-greedy sequences in subdomains that need not be Lipschitz.
The proof of Proposition~\ref{prop:matern-separation-radius} uses the width
inequality in \eqref{eq:matern-width-summary}.
Combining~\eqref{eq:matern-innovation} with
Lemma~\ref{lem:residual-envelope-domination} and
Proposition~\ref{prop:polynomial-transfer} gives the following regret rate.

\begin{theorem}[Simple-regret rate for EI with a Mat\'ern kernel]
\label{thm:matern-rate}
In the setting of Section~\ref{sec:setting}, let
$k=k_\nu$ be as in Proposition~\ref{prop:matern-separation-radius}.  Then,
along every fixed-prior EI trajectory satisfying~\eqref{eq:approx-ei},
\begin{equation}
    r_T
    =O\bigl(N^{-\nu/d}\bigr),
    \qquad N=T-n_0\longrightarrow\infty.
  \label{eq:matern-rate}
\end{equation}
For fixed $B$, $\X$, $k$, $\sigma$, $\eta$, and initial design, neither the
implied constant nor the starting index depends on the choice of
$f\in\Hk$ with $\lVert f\rVert_{\Hk}\leq B$.
\end{theorem}

\begin{proof}
See Appendix~\ref{app:matern-rate-proof}.
\end{proof}

\subsection{From exponential bounds on the separation radii to regret}

\begin{proposition}[Regret from exponential bounds on the separation radii]
\label{prop:log-enhanced-transfer}
Suppose that, for some $A,b>0$ and all sufficiently large~$m$,
\begin{equation}
  \delta_m(k_0,\X)
  \leq A\exp\{-b m^{1/d}\log m\}.
  \label{eq:log-enhanced-envelope}
\end{equation}
Let $T>n_0$, put $N:=T-n_0$, and define
$$
  \psi_d(N):=\min\{N,N^{1/d}\log(eN)\}.
$$
Then there are constants $A',b'>0$ and an integer $N_0\geq1$ such that, for
every fixed-prior EI trajectory satisfying~\eqref{eq:approx-ei},
\begin{equation}
  r_T
  \leq A'\exp\{-b'\psi_d(N)\},
  \label{eq:log-enhanced-rate}
\end{equation}
for every $N\geq N_0$.
We write $a_n\asymp b_n$ if there are constants
$0<c\leq C<\infty$ such that
$c b_n\leq a_n\leq C b_n$ for all sufficiently large~$n$.
Here $\psi_1(N)=N$, whereas, for $d\geq2$ and all sufficiently large~$N$,
$\psi_d(N)=N^{1/d}\log(eN)\asymp N^{1/d}\log N$.
For fixed $f$, $A',b'$ and $N_0$ may depend on
$A,b,d,r_{n_0},q,C$ and on the rank from which
\eqref{eq:log-enhanced-envelope} holds.  If $f$ ranges over the ball
$\{h\in\Hk:\lVert h\rVert_{\Hk}\leq B\}$, the constants $A',b'$ and the
starting index $N_0$ can instead be chosen uniformly in $f$.  This uniform
choice may depend on $B,\sigma,\eta,A,b,d,k,\X$, the fixed initial design,
and the rank from which \eqref{eq:log-enhanced-envelope} holds.
\end{proposition}

\begin{proof}
See Appendix~\ref{app:log-enhanced-transfer-proof}.
\end{proof}

\paragraph{Proof idea.}
Apply Theorem~\ref{thm:order-statistic-transfer} with $L=1$ when $q=0$.
When $0<q<1$, take $L=\lfloor N/2\rfloor$ for $d=1$ and
$L=\lceil N^{1/d}\log(eN)\rceil$ for $d\geq2$.  Then
$q^Lr_{n_0}$ and $\delta_{N-L+1}(k_0,\X)$ are both
$O(\exp[-c\psi_d(N)])$ for some $c>0$.
Equation~\eqref{eq:order-transfer-coarse} gives
\eqref{eq:log-enhanced-rate}.

\subsection{Squared-exponential kernels}

\begin{proposition}[Squared-exponential separation-radius decay]
\label{prop:se-separation-radius}
Let $\X\subset\R^d$ be nonempty and compact, and fix $\varrho>0$.  Let $k$ be the
restriction to $\X\times\X$ of the unit-diagonal isotropic squared-exponential
kernel on $\R^d$ defined by
$$
  (x,y)\longmapsto
  \exp\left(-\frac{\lVert x-y\rVert^2}{2\varrho^2}\right).
$$
There are constants $A,b>0$ such that, for all sufficiently large~$m$,
\begin{equation}
  \delta_m(k,\X)
  \leq A\exp\bigl(-b m^{1/d}\log m\bigr).
  \label{eq:se-innovation}
\end{equation}
When $d=1$, the following sharper estimate holds, with
$\log 0=-\infty$:
\begin{equation}
  \limsup_{m\to\infty}
  \frac{\log\delta_m(k,\X)}{m\log m}
  \leq-\frac12.
  \label{eq:se-one-dimensional-leading-constant}
\end{equation}
\end{proposition}

\begin{proof}
See Appendix~\ref{app:se-separation-radius}.
\end{proof}

\paragraph{Proof idea.}
Enclose $\X$ in a cube $\mathcal Q$.  For a finite nonempty set
$\mathcal Y\subset\mathcal Q$ with small fill distance
$h_{\mathcal Y,\mathcal Q}$,
\citet[Theorem~11.22 and Equation~(11.11)]{wendland2005} gives a uniform
power-function bound on $\mathcal Q$ of the form
$A_0\exp\{-b_0|\log h_{\mathcal Y,\mathcal Q}|/
h_{\mathcal Y,\mathcal Q}\}$ for some
$A_0,b_0>0$.
Taking $\mathcal Y$ to be a Cartesian grid with at most $m$ points gives
$h_{\mathcal Y,\mathcal Q}=O(m^{-1/d})$, and the span of its kernel
translates has
dimension at most $m$.  The restriction $h\mapsto h|_\X$ from the RKHS of
the squared-exponential kernel on $\mathcal Q$ to $\mathcal H_k$ is
contractive and maps this span to a subspace of dimension at most $m$.  The
restricted subspace has worst-case approximation error over $\X$ no larger
than that of the original subspace over $\mathcal Q$.  Hence the power-function
estimate gives
$d_m(k,\X)=O(\exp[-c' m^{1/d}\log m])$.  The compact set $\X$ need not have
interior or a regular boundary.  Corollary~\ref{cor:width-envelopes} then
gives~\eqref{eq:se-innovation}.
In dimension one, \citet[Theorem~2]{yarotsky2013} gives a
conditional-variance bound, uniform over $\X$, after conditioning on any set
of distinct points.  Taking square roots and applying the resulting
power-function bound to the final innovation norm of each tuple of distinct
points gives~\eqref{eq:se-one-dimensional-leading-constant}.  A tuple with a
repeated point has a zero innovation.

\begin{theorem}[Simple-regret rates for EI with a squared-exponential kernel]
\label{thm:se-rate}
In the setting of Section~\ref{sec:setting}, let $k$ be the
squared-exponential kernel in
Proposition~\ref{prop:se-separation-radius}.  Let $T>n_0$ and put
$N:=T-n_0$.
\begin{enumerate}
  \item In every dimension $d\geq1$, for every fixed-prior EI trajectory
  satisfying the weak-EI condition~\eqref{eq:approx-ei},
  there are $A',b'>0$ such that
  \begin{equation}
    r_T
    \leq
    A'\exp\{-b'\min\{N,N^{1/d}\log(eN)\}\},
  \label{eq:se-rate}
  \end{equation}
for all sufficiently large~$N$.
  \item When $d=1$ and EI is maximized exactly, for every $0<b''<1/2$ there
  is $A''>0$ such that
  \begin{equation}
    r_T\leq A''\exp\{-b''N\log(eN)\},
  \label{eq:se-one-dimensional-exact-rate}
  \end{equation}
for every $N\geq1$.
\end{enumerate}
For fixed $B$, $\X$, $k$, $\sigma$, $\eta$, and the initial design, the
constants and starting index in~\eqref{eq:se-rate} do not depend on the
choice of $f\in\Hk$ with $\lVert f\rVert_{\Hk}\leq B$.  When $d=1$ and EI
is maximized exactly, the same holds for $A''$ when $b''$ is fixed.
\end{theorem}

\paragraph{Proof idea.}
Combining~\eqref{eq:se-innovation} with
Lemma~\ref{lem:residual-envelope-domination} and
Proposition~\ref{prop:log-enhanced-transfer} gives~\eqref{eq:se-rate}.  In
dimension one, interpolation makes EI zero at queried points, whereas strict
positive definiteness and~\eqref{eq:ei-closed-form} make it positive at
unqueried points.  Hence the exact~EI policy selects an unqueried point whenever
one remains.
Proposition~\ref{prop:yarotsky-power} then bounds
$\sup_{x\in\X}s_{T-1}(x)$.  With $\eta=1$,
\eqref{eq:optimizer-uncertainty-recursion} gives
$$
  r_T\leq(B_0+C)\sup_{x\in\X}s_{T-1}(x),
$$
which yields~\eqref{eq:se-one-dimensional-exact-rate}.

\begin{proof}
See Appendix~\ref{app:se-rate-proof}.
\end{proof}

\begin{remark}[Logarithmic loss in dimension one]
\label{rem:one-dimensional-bottleneck}
Suppose $d=1$ and $0<q<1$, and choose $1\leq L\leq N$ in the transfer
bound.  Write $M:=N-L+1$ for the separation-radius index.  The
separation-radius term is
bounded by $A\exp(-bM\log M)$ in~\eqref{eq:log-enhanced-envelope}, whereas
the contraction term is $q^{N-M+1}$.  The proof of
Proposition~\ref{prop:log-enhanced-transfer} used $L=\lfloor N/2\rfloor$ to
establish the rate.  Balancing the two terms more closely by choosing
$M\asymp N/\log(eN)$ and $L=N-M+1$ in
\eqref{eq:order-transfer-coarse} gives an upper bound $O(e^{-cN})$.
Thus~\eqref{eq:order-transfer-coarse} loses the factor $\log(eN)$ because
$q^L$ decreases only geometrically in $L$.
The refinement for exact maximization of EI
in~\eqref{eq:se-one-dimensional-exact-rate} combines the power-function bound
for every set of distinct points in dimension one with a bound for
$s_n(x^\star)$.
\end{remark}

\section{Minimax comparison}
\label{sec:minimax-comparison}

The preceding upper bounds hold on every compact domain.  The minimax
comparisons below assume that $\X$ has nonempty interior.

\subsection{Minimax classes}

In the minimax notation below, $N\geq1$ denotes the total number of
evaluations, including the initial design.  Let $\mathcal A_N$ be the class
of deterministic algorithms $\underline X_N=(X_1,\ldots,X_N)$, where
$X_1\in\X$ is chosen as part of the algorithm and, for $2\leq j\leq N$,
$$
  X_j:(\X\times\mathbb R)^{j-1}\longrightarrow\X
$$
is Borel measurable.  For an objective function $f$, the query points are
defined recursively by $X_1(f)=X_1$ and
$$
  X_j(f)
  =
  X_j\left(
    \bigl(X_i(f),f(X_i(f))\bigr)_{i=1}^{j-1}
  \right),
  \qquad 2\leq j\leq N.
$$
The class $\mathcal A_\infty$ consists of infinite deterministic sequential
algorithms $\underline X=(X_n)_{n\geq1}$ whose truncation
$(X_1,\ldots,X_N)$ belongs to $\mathcal A_N$ for every $N\geq1$.

A deterministic method with $N$ evaluations consists of
$\underline X_N\in\mathcal A_N$ and a Borel final-recommendation rule
$$
  \widehat X_N:(\X\times\mathbb R)^N\longrightarrow\X.
$$
For an objective function $f$, write
$$
  \widehat x_N(f)
  :=
  \widehat X_N\left(
    \bigl(X_j(f),f(X_j(f))\bigr)_{j=1}^N
  \right).
$$
The recommendation $\widehat x_N(f)$ need not be one of the query points.

\begin{definition}[Deterministic minimax loss]
For $B>0$, the deterministic minimax loss is
\begin{equation}
  \mathcal R_N^{\mathrm{det}}(B;k,\X)
  :=
  \inf_{\substack{
    \underline X_N\in\mathcal A_N\\
    \widehat X_N\ \text{Borel}
  }}
  \sup_{\lVert f\rVert_{\Hk}\leq B}
  \left\{
    f\bigl(\widehat x_N(f)\bigr)
    -\min_{x\in\X}f(x)
  \right\}.
  \label{eq:deterministic-minimax-loss}
\end{equation}
\end{definition}

Let $\widetilde{\mathcal A}_N$ be the class of sequential
strategies $\underline X$ for which each query point is a measurable function
of the preceding query-value history and of possible internal randomness that
is independent of $f$.  After $N$ exact evaluations, the strategy returns a
recommendation $\widehat x_N\in\X$ that is a measurable function of the
complete history and of the same internal randomness.  The recommendation
need not be one of the query points.  Write $\mathsf E_{\underline X,\,f}$ for
expectation over that randomness when the objective function is $f$.
The class $\widetilde{\mathcal A}_\infty$ consists of infinite randomized
sequential strategies whose truncation after $N$ evaluations belongs to
$\widetilde{\mathcal A}_N$ for every $N\geq1$.

\begin{definition}[Randomized minimax loss]
For $B>0$, the corresponding minimax expected loss is
\begin{equation}
  \mathcal R_N^{\mathrm{rand}}(B;k,\X)
  :=\inf_{\underline X\in\widetilde{\mathcal A}_N}
    \sup_{\lVert f\rVert_{\Hk}\leq B}
    \mathsf E_{\underline X,\,f}\left[
      f(\widehat x_N)-\min_{x\in\X}f(x)
    \right].
  \label{eq:randomized-recommendation-minimax}
\end{equation}
\end{definition}

\paragraph{Worst-case regret of exact~EI.}
Let
$$
  \underline X^{\mathrm{EI}}
  =(X_n^{\mathrm{EI}})_{n\geq1}
$$
denote a measurable exact~EI policy in $\mathcal A_\infty$ with the fixed
initial design, as provided by
Proposition~\ref{prop:measurable-exact-ei}.  Define, for $N\geq n_0$ and
$B>0$, its worst-case regret by
\begin{equation}
  R_N^{\mathrm{EI}}(B;k,\X)
  :=
  \sup_{\lVert f\rVert_{\Hk}\leq B}
  \left\{
    \min_{1\leq j\leq N}f\bigl(X_j^{\mathrm{EI}}(f)\bigr)
    -\min_{x\in\X}f(x)
  \right\}.
  \label{eq:ei-policy-worst-case-regret}
\end{equation}
This is the worst-case loss of the deterministic method that follows the EI
query rules and recommends the earliest query point at which the best observed
value is attained.

\subsection{Mat\'ern kernels}

\begin{theorem}[Minimax optimality for Mat\'ern kernels]
\label{thm:matern-minimax}
Let $\X\subset\R^d$ be compact with nonempty interior, let $k_\nu$ be the
Mat\'ern kernel in Proposition~\ref{prop:matern-separation-radius}, and let
$B>0$.  There are constants
$c_{\mathrm M},C_{\mathrm M}>0$ such that, for all sufficiently large~$N$,
\begin{equation}
  c_{\mathrm M}N^{-\nu/d}
  \leq \mathcal R_N^{\mathrm{rand}}(B;k_\nu,\X)
  \leq \mathcal R_N^{\mathrm{det}}(B;k_\nu,\X)
  \leq R_N^{\mathrm{EI}}(B;k_\nu,\X)
  \leq C_{\mathrm M}N^{-\nu/d}.
  \label{eq:matern-minimax-chain}
\end{equation}
\end{theorem}

\begin{proof}
The classical bump-function argument used in the proof of
\citet[Theorem~1]{bull2011} applies to $k_\nu$ and gives the lower bound for
$\mathcal R_N^{\mathrm{rand}}$ in~\eqref{eq:matern-minimax-chain}.
Every deterministic method in the definition of
$\mathcal R_N^{\mathrm{det}}$ belongs to
$\widetilde{\mathcal A}_N$ by using no internal randomness.  Hence
$\mathcal R_N^{\mathrm{rand}}(B;k_\nu,\X)
\leq\mathcal R_N^{\mathrm{det}}(B;k_\nu,\X)$.

Consider the method that uses the first $N$ decision rules of
$\underline X^{\mathrm{EI}}$ and recommends the earliest query point at which
the best observed value is attained.  This recommendation is a Borel function
of the finite history, and its loss equals the regret in
\eqref{eq:ei-policy-worst-case-regret}.  Therefore
$\mathcal R_N^{\mathrm{det}}(B;k_\nu,\X)
\leq R_N^{\mathrm{EI}}(B;k_\nu,\X)$.
Finally, the uniform bound in Theorem~\ref{thm:matern-rate} gives
$R_N^{\mathrm{EI}}(B;k_\nu,\X)
=O((N-n_0)^{-\nu/d})=O(N^{-\nu/d})$, since $n_0$ is fixed.
\end{proof}

\subsection{Squared-exponential kernels}

\begin{theorem}[Minimax optimality for the squared-exponential kernel]
\label{thm:se-minimax}
Let~$\X\subset\R^d$ be compact with nonempty interior, let $k$ be as in
Proposition~\ref{prop:se-separation-radius}, and let $B>0$.  There are positive
constants $c_1,c_2,C_1,C_2$ such that, for all sufficiently large~$N$,
\begin{equation}
  c_1\exp\{-C_1N^{1/d}\log(eN)\}
  \leq \mathcal R_N^{\mathrm{det}}(B;k,\X)
  \leq R_N^{\mathrm{EI}}(B;k,\X)
  \leq C_2\exp\{-c_2N^{1/d}\log(eN)\}.
  \label{eq:se-minimax-scale}
\end{equation}
The constants $c_1,C_1$ depend only on $B$, $k$, and $\X$.  The constants
$c_2,C_2$ may also depend on $\sigma$ and the fixed initial design.
\end{theorem}

\paragraph{Proof idea.}
After restriction to a cube contained in $\operatorname{int}(\X)$ and
rescaling to $[0,1]^d$, the metric-entropy estimate in
\citet[Theorem~3]{kuhn2011}, combined with
\citet[Theorem~5.1]{xu2024}, gives a lower bound on the number of evaluations
required to make the simple regret computed from the best observed value
smaller than a prescribed threshold.  Appending the final recommendation of a
deterministic method as the $(N+1)$st query gives the lower bound for
$\mathcal R_N^{\mathrm{det}}$ in~\eqref{eq:se-minimax-scale}.

For the upper bound, use the measurable exact~EI policy and recommend the
earliest query point at which the best observed value is attained.
Theorem~\ref{thm:se-rate} gives the required bound, using
\eqref{eq:se-rate} when $d\geq2$ and
\eqref{eq:se-one-dimensional-exact-rate} when $d=1$.

\begin{proof}
See Subsection~\ref{app:se-minimax-proof} for the complete proof.
\end{proof}

\section{Discussion}
\label{sec:scope}

Fixed-prior EI trajectories generated by weak-EI policies are simple-regret
consistent for every continuous positive-semidefinite kernel on a nonempty
compact domain
(Corollary~\ref{cor:ei-consistency}).  For Mat\'ern kernels, their rate is
$N^{-\nu/d}$ for every $\nu>0$.  When $B>0$ and $\X$ has nonempty interior,
Theorem~\ref{thm:matern-minimax} shows that the exact~EI policy is
minimax-rate optimal among deterministic methods whose final recommendation
may be any point of $\X$.
Thus the exact fixed-prior~EI policy attains the Mat\'ern minimax rate for
every $\nu>0$ without space-filling random queries, which answers the question
that motivated this article.

Although the squared-exponential kernel does not have the NEB property and
exact~EI trajectories need not be dense
\citep[Theorem~3]{yarotsky2013},
Theorems~\ref{thm:se-rate} and~\ref{thm:se-minimax} also show that the
exact~EI policy is minimax-rate optimal up to constants in the exponent among
deterministic methods.

The main idea is to represent each query point $x_i$ by its evaluation
representer $k(x_i,\cdot)$ and thereby reduce the upper-bound analysis of EI
to geometric approximation in $\Hk$.  The comparison between sequential
separation radii and Kolmogorov widths is inspired by work on greedy
approximation \citep{binev2011,devore2013,wenzelSantinHaasdonk2023}.
In the proof of Lemma~\ref{lem:width-product}, a subspace dimension in the
Kolmogorov-width bound is chosen separately for each eigenvalue of the Gram
matrix of an ordered tuple of evaluation representers.
Lemma~\ref{lem:innovation-order-statistics} bounds the ranked selected-point
innovation norms over each finite horizon by the separation radii of $k_0$,
and hence by those of $k$.
The finite-budget argument in
Theorem~\ref{thm:order-statistic-transfer} uses
$r_{n+1}\leq q r_n+C v_{n+1}$ at the evaluations with the smallest
innovation norms
and uses regret monotonicity at all remaining iterations.  For both kernels,
the estimates of $d_m(k,\X)$ follow from standard power-function estimates in
scattered data approximation.

The analysis assumes exact observations and covariance parameters fixed
independently of the observations.  Noise and covariance parameters selected
from the data lie outside its scope.

\appendix

\section{Consistency and density}
\label{app:consistency-density-proofs}

The compactness argument of \citet[Lemma~12]{vazquezBect2010} yields a
subsequence used in both the consistency and density proofs.

\begin{lemma}[Vanishing EI along a subsequence]
\label{lem:cluster-selected-ei}
In the setting of Section~\ref{sec:setting}, let
$(x_i)_{i\geq1}\subset\X$ be any infinite sequence of query points, with the
corresponding values of $f$ observed exactly.  Then there exist indices
$n_j\to\infty$ such that
\begin{equation}
  \EI_{n_j}(x_{n_j+1})\longrightarrow0.
  \label{eq:cluster-selected-ei}
\end{equation}
\end{lemma}

\begin{proof}
By compactness, there are $\bar x\in\X$ and strictly increasing indices
$(\varphi_j)$ such that $x_{\varphi_j}\to\bar x$.  After
discarding finitely many terms and reindexing the subsequence, we may assume
$\varphi_1-1\geq n_0$ and put $n_j:=\varphi_j-1$.  For $j\geq2$,
$x_{\varphi_{j-1}}$ has already been observed at time $n_j$.  Hence
\begin{align*}
  s_{n_j}(x_{n_j+1})
  &\leq
  \bigl\|k(\cdot,x_{\varphi_j})
       -k(\cdot,x_{\varphi_{j-1}})\bigr\|_{\Hk}
  \longrightarrow0,\\
  I_{n_j}(x_{n_j+1})
  &\leq
  \bigl(f(x_{\varphi_{j-1}})-f(x_{\varphi_j})\bigr)_+
  \longrightarrow0.
\end{align*}
The first inequality holds because
$k(\cdot,x_{\varphi_{j-1}})\in V_{n_j}$.  The second uses
$m_{n_j}\leq f(x_{\varphi_{j-1}})$.  Both right-hand sides tend to zero
because $x_{\varphi_j}\to\bar x$, the map
$x\mapsto k(\cdot,x)$ is continuous, and $f$ is continuous.  The upper bound
in~\eqref{eq:ei-comparison} then gives~\eqref{eq:cluster-selected-ei}.
\end{proof}

\paragraph{Consequence for consistency.}
For a fixed-prior EI trajectory satisfying~\eqref{eq:approx-ei}, the lower
bound in~\eqref{eq:ei-comparison} at a minimizer $x^\star$ and
\eqref{eq:approx-ei} give
$$
  \eta\kappa_a r_n
  \leq \EI_n(x_{n+1}).
$$
Along the subsequence from Lemma~\ref{lem:cluster-selected-ei}, the right-hand
side tends to zero.  Since $\eta\kappa_a>0$, $r_{n_j}\to0$.  The
nonnegative sequence $(r_n)$ is nonincreasing, so $r_n\to0$.

\paragraph{Density under the NEB property.}

Because $s_n(x)=P_{(x_1,\ldots,x_n)}(x)$, the density proof applies the NEB
property to the original kernel $k$.

\begin{proof}[Proof of Theorem~\ref{thm:neb-density}]
Suppose that the query points are not dense.  Then there is an
$x\in\X\setminus\overline{\{x_i,\,i\geq1\}}$.  The NEB property gives
$s_n(x)\not\to0$.  Since the spaces
$V_n$ are nested, the normalized posterior standard deviations $s_n(x)$ are
nonincreasing in $n$.  Hence
$$
  0<\inf_{j\geq n_0}s_j(x)\leq s_n(x)\leq1,
  \qquad n\geq n_0.
$$
The RKHS error estimate~\eqref{eq:rkhs-error} bounds $\mu_n(x)$ uniformly in
$n$, and continuity of $f$ on the compact set $\X$ bounds $m_n$.  Hence the
pairs
$$
  \bigl(m_n-\mu_n(x),s_n(x)\bigr)
$$
lie in a compact subset of $\mathbb R\times(0,1]$.  With
$u=m_n-\mu_n(x)$ and $s=s_n(x)$, the positive-$s$ branch of
\eqref{eq:ei-closed-form} is the function
$$
  (u,s)\longmapsto \sigma s\,
  \tau\left(\frac{u}{\sigma s}\right).
$$
This map is continuous and strictly positive, so it has a
positive minimum on the compact subset containing these pairs.  Consequently
$\inf_{n\geq n_0}\EI_n(x)>0$.  The selection condition
\eqref{eq:approx-ei} then implies that $\EI_n(x_{n+1})$ is bounded away from
zero, contradicting
Lemma~\ref{lem:cluster-selected-ei}.
\end{proof}

\section{Using the normalized posterior standard deviation at a minimizer}
\label{app:optimizer-uncertainty}

\begin{lemma}[Refined one-step regret bound]
\label{lem:optimizer-uncertainty}
Let $x^\star$ be a minimizer of $f$.  For every $n\geq n_0$, if the selected
point $x_{n+1}$ satisfies~\eqref{eq:approx-ei}, then
\begin{equation}
  r_{n+1}
  \leq
  \min\left\{
    q r_n,\,
    (1-\eta)r_n+\eta B_0 s_n(x^\star)
  \right\}
  +C s_n(x_{n+1}).
  \label{eq:optimizer-uncertainty-recursion}
\end{equation}
\end{lemma}

\begin{proof}
Since $I_n(x^\star)=r_n$, the two lower bounds in
\eqref{eq:ei-comparison} give
$$
  \EI_n(x^\star)
  \geq
  \max\{r_n-B_0s_n(x^\star),\kappa_a r_n\}.
$$
At the selected point,
$I_n(x_{n+1})=r_n-r_{n+1}$.  The upper bound in
\eqref{eq:ei-comparison} and the selection condition
\eqref{eq:approx-ei} therefore give
$$
  \eta\max\{r_n-B_0s_n(x^\star),\kappa_a r_n\}
  \leq r_n-r_{n+1}+Cs_n(x_{n+1}).
$$
Rearranging gives
$$
  r_{n+1}
  \leq
  \min\left\{
    (1-\eta)r_n+\eta B_0s_n(x^\star),\,
    (1-\eta\kappa_a)r_n
  \right\}
  +Cs_n(x_{n+1}).
$$
Since $q=1-\eta\kappa_a$, this is
\eqref{eq:optimizer-uncertainty-recursion}.
\end{proof}

The proof of Theorem~\ref{thm:order-statistic-transfer} uses only the
$q r_n$ branch.  Using the second branch in that argument would require
additional control of $s_n(x^\star)$, whereas
Lemma~\ref{lem:innovation-order-statistics} bounds only the selected-point
standard deviations.  The proof for exact~EI when $d=1$ uses a uniform bound
on the power function to bound $s_n(x^\star)$.

\section{Ranked innovation norms and finite-budget bounds}
\label{app:transfer-refinements}

\subsection{Proof of the ranked innovation-norm bound}
\label{app:ranked-innovation-proof}

\begin{proof}[Proof of Lemma~\ref{lem:innovation-order-statistics}]
Fix $1\leq j\leq N$, and let $\mathcal I_j\subseteq\{1,\ldots,N\}$, with
$|\mathcal I_j|=j$, contain indices of the $j$ largest values among
$\widetilde v_1,\ldots,\widetilde v_N$, with ties resolved arbitrarily.
Enumerate this set in chronological order as
$$
  \mathcal I_j=\{i_1<\cdots<i_j\}.
$$
The chronological enumeration need not arrange the values
$(\widetilde v_{i_h})_{h=1}^j$ in decreasing order.
Put $y_h:=x_{n_0+i_h}$ for $1\leq h\leq j$.  Then
\begin{equation}
  \operatorname{span}\{\Psi_0(x_{n_0+i}):1\leq i<i_h\}
  \supseteq
  \operatorname{span}\{\Psi_0(y_\ell),\,1\leq\ell<h\},
  \label{eq:ranked-predecessor-span-inclusion}
\end{equation}
where the span on the right is $\{0\}$ when $h=1$.
Equations~\eqref{eq:post-initial-innovation-reindexing} and
\eqref{eq:posterior-restart-power}, together
with~\eqref{eq:ranked-predecessor-span-inclusion}, give
\begin{equation}
\begin{aligned}
  \widetilde v_{i_h}
  &=
  \dist_{\Hk}\left(
    \Psi_0(y_h),
    \operatorname{span}\{\Psi_0(x_{n_0+i}):1\leq i<i_h\}
  \right)\\
  &\leq
  \dist_{\Hk}\left(
    \Psi_0(y_h),
    \operatorname{span}\{\Psi_0(y_\ell),\,1\leq\ell<h\}
  \right),
  \qquad 1\leq h\leq j.
\end{aligned}
  \label{eq:ranked-distance-comparison}
\end{equation}
Taking the minimum over $h$ in~\eqref{eq:ranked-distance-comparison} and then
using the definition of $\delta_j$, with $k_0(\cdot,y)=\Psi_0(y)$ and
$\mathcal H_{k_0}$ carrying the norm inherited from $\Hk$, gives
\begin{align*}
  v_{[j]}
  &=\min_{1\leq h\leq j}\widetilde v_{i_h}\\
  &\leq
  \min_{1\leq h\leq j}
  \dist_{\Hk}\left(
    \Psi_0(y_h),
    \operatorname{span}\{\Psi_0(y_\ell),\,1\leq\ell<h\}
  \right)\\
  &\leq\delta_j(k_0,\X).
\end{align*}
Lemma~\ref{lem:residual-envelope-domination} completes
\eqref{eq:innovation-order-statistics}.

\end{proof}

\subsection{Proof of the recursion along a subsequence}
\label{app:capped-certificate-proof}

\begin{proof}[Proof of Lemma~\ref{lem:capped-certificate}]
The base case is $u_0\leq U_0$.  Suppose that
$u_{i-1}\leq U_{i-1}$.  The two assumed inequalities at index $i$ give
$$
  u_i
  \leq\min\{u_{i-1},qu_{i-1}+Cv_i\}
  =F_{v_i}(u_{i-1}).
$$
Because $q\geq0$, the map $F_{v_i}$ is nondecreasing.  Hence
$$
  u_i\leq F_{v_i}(u_{i-1})
      \leq F_{v_i}(U_{i-1})=U_i,
$$
which proves $u_i\leq U_i$ for every $i$.

Let $S\subseteq\{1,\ldots,N\}$ and set $s:=|S|$.  When $s\geq1$, list its
elements as $S=\{j_1<\cdots<j_s\}$.  Define
$$
  A(S):=q^sU_0+C\sum_{h=1}^sq^{s-h}v_{j_h},
$$
where the sum is zero when $S=\varnothing$.  Thus
$A(\varnothing)=U_0$.
We prove by induction on $i$ that
$$
  U_i=\min_{S\subseteq\{1,\ldots,i\}}A(S).
$$
For $i=0$, the only subset is $\varnothing$, so the right-hand side is
$A(\varnothing)=U_0$.  Suppose that the claim holds at index $i-1$, and
partition the subsets of $\{1,\ldots,i\}$ according to whether they contain
$i$.  For every $S'\subseteq\{1,\ldots,i-1\}$, the definition gives
$$
  A(S'\cup\{i\})=qA(S')+Cv_i.
$$
Since $q\geq0$, the induction hypothesis therefore yields
\begin{align*}
  \min_{S\subseteq\{1,\ldots,i\}}A(S)
  &=\min\left\{
      \min_{S'\subseteq\{1,\ldots,i-1\}}A(S'),
      \min_{S'\subseteq\{1,\ldots,i-1\}}A(S'\cup\{i\})
    \right\}\\
  &=\min\{U_{i-1},qU_{i-1}+Cv_i\}
   =U_i.
\end{align*}
Taking $i=N$ proves~\eqref{eq:pathwise-subset-certificate}.
\end{proof}

\subsection{Proof of the finite-budget bound}
\label{app:ranked-transfer-proof}

\begin{proof}[Proof of Theorem~\ref{thm:order-statistic-transfer}]
As in Lemma~\ref{lem:innovation-order-statistics}, reindex the post-initial
innovation norms from $1$ to $N$ by setting
$$
  \widetilde v_i
  :=v_{n_0+i}
  =s_{n_0+i-1}(x_{n_0+i}),
  \qquad 1\leq i\leq N,
$$
and let $v_{[1]}\geq\cdots\geq v_{[N]}$ be their decreasing rearrangement.
Set $u_i:=r_{n_0+i}$, so that $u_N=r_T$.  The hypotheses of
Theorem~\ref{thm:order-statistic-transfer} give
$$
  u_i\leq u_{i-1},
  \qquad
  u_i\leq q u_{i-1}+C\widetilde v_i,
  \qquad 1\leq i\leq N.
$$
Define the auxiliary sequence $(U_i)_{i=0}^N$ by
\eqref{eq:pathwise-capped-certificate}, starting from
$U_0:=u_0=r_{n_0}$.  The inequalities above verify the hypotheses of
Lemma~\ref{lem:capped-certificate}, which gives $u_N\leq U_N$ and the
minimization formula
\eqref{eq:pathwise-subset-certificate} for $U_N$.

Choose the indices of the $L$ smallest values among
$\widetilde v_1,\ldots,\widetilde v_N$, resolving ties arbitrarily.  Write them
as $1\leq i_1<\cdots<i_L\leq N$ and set
$S_\star:=\{i_1,\ldots,i_L\}$.  The sequence
$(\widetilde v_{i_h})_{h=1}^L$ is in chronological order and, since
$M=N-L+1$, has decreasing rearrangement
$$
  (v_{[M]},v_{[M+1]},\ldots,v_{[N]}).
$$
The minimization in~\eqref{eq:pathwise-subset-certificate} ranges over every
subset $S$, including $S_\star$.  Hence $U_N$ is at most the value obtained
with $S=S_\star$.  Since $|S_\star|=L$, this gives
$$
  r_T=u_N\leq U_N
  \leq q^Lr_{n_0}
  +C\sum_{h=1}^Lq^{L-h}\widetilde v_{i_h}.
$$
Pairing the larger selected values with the larger weights maximizes the sum,
so the largest selected value receives weight $1$, the next largest weight
$q$, and so on.
Applying Lemma~\ref{lem:innovation-order-statistics} to each
$v_{[M+j]}$ then gives
$$
  \sum_{h=1}^Lq^{L-h}\widetilde v_{i_h}
  \leq\sum_{j=0}^{L-1}q^jv_{[M+j]}
  \leq\sum_{j=0}^{L-1}q^j\delta_{M+j}(k_0,\X).
$$
This proves~\eqref{eq:order-transfer-refined}.  Since the separation radii
are nonincreasing,
$$
  \sum_{j=0}^{L-1}q^j\delta_{M+j}(k_0,\X)
  \leq
  \delta_M(k_0,\X)\sum_{j=0}^{L-1}q^j
  =\frac{1-q^L}{1-q}\delta_M(k_0,\X),
$$
which proves~\eqref{eq:order-transfer-coarse}.
\end{proof}

\section{Kolmogorov widths, Gram determinants, and separation radii}
\label{app:width-details}

\subsection{Operator form of the Eckart--Young theorem}
\label{app:finite-domain-eckart-young}

\begin{lemma}[Eckart--Young identity]
\label{lem:finite-domain-eckart-young}
Let $\mathcal H$ be a real Hilbert space, let $m\geq1$ be an integer, and let
$S:\mathbb R^m\to\mathcal H$ be linear.  Denote the eigenvalues of $S^*S$,
counted with multiplicity, by
$\lambda_1\geq\cdots\geq\lambda_m\geq0$.  For a linear map
$A:\mathbb R^m\to\mathcal H$, its squared Hilbert--Schmidt norm is
$$
  \lVert A\rVert_{\mathrm{HS}}^2
  :=\sum_{\ell=1}^m\lVert Ae_\ell\rVert_{\mathcal H}^2,
$$
where $(e_\ell)_{\ell=1}^m$ is the standard basis of $\mathbb R^m$.  For every
integer $0\leq i\leq m$,
\begin{equation}
  \inf_{\substack{A:\mathbb R^m\to\mathcal H\ \mathrm{linear}\\
                  \operatorname{rank}(A)\leq i}}
    \lVert S-A\rVert_{\mathrm{HS}}^2
  =\sum_{\ell>i}\lambda_\ell.
  \label{eq:finite-domain-eckart-young}
\end{equation}
The infimum is attained by truncating a singular-value decomposition of $S$
after $i$ terms.
\end{lemma}

\begin{proof}
If $S=0$, the conclusion is immediate.  Assume henceforth that $S\neq0$.
Let $\mathcal S:=\operatorname{ran}S$, which is finite-dimensional and
therefore closed, and let $\Pi_{\mathcal S}$ be its orthogonal projector.  For
every linear operator $A:\mathbb R^m\to\mathcal H$, orthogonality of
$\mathcal S$ and $\mathcal S^\perp$ gives
\begin{equation}
  \lVert S-A\rVert_{\mathrm{HS}}^2
  =
  \lVert S-\Pi_{\mathcal S}A\rVert_{\mathrm{HS}}^2
  +
  \lVert(I-\Pi_{\mathcal S})A\rVert_{\mathrm{HS}}^2.
  \label{eq:eckart-young-range-projection}
\end{equation}
Moreover,
$\operatorname{rank}(\Pi_{\mathcal S}A)\leq\operatorname{rank}(A)$.  Hence the
infimum in~\eqref{eq:finite-domain-eckart-young} is unchanged if it is
restricted to operators whose range lies in $\mathcal S$.

If $B\in\mathbb R^{p\times m}$, with $p\leq m$, let
$\lambda_1(B^\top B)\geq\cdots\geq\lambda_m(B^\top B)\geq0$ be the
eigenvalues of $B^\top B$.  With $\lVert\cdot\rVert_{\mathrm F}$ denoting the
Frobenius norm, the matrix Eckart--Young theorem
\citep{eckartYoung1936} gives
\begin{equation}
  \min_{\substack{\widetilde B\in\mathbb R^{p\times m}\\
                   \operatorname{rank}(\widetilde B)\leq i}}
    \lVert B-\widetilde B\rVert_{\mathrm F}^2
  =\sum_{\ell>i}\lambda_\ell(B^\top B),
  \qquad 0\leq i\leq m.
  \label{eq:matrix-eckart-young}
\end{equation}
If $B=U\Sigma V^\top$ is a singular-value decomposition, a minimizer is
$U\Sigma_iV^\top$, where $\Sigma_i$ retains the
$\min\{i,p\}$ largest singular values and replaces the others by zero.

Put $p:=\dim\mathcal S\leq m$, and choose orthonormal bases of $\mathbb R^m$
and $\mathcal S$.  In these bases, operators from $\mathbb R^m$ to
$\mathcal S$ are represented by $p\times m$ matrices, and the
Hilbert--Schmidt norm agrees with the Frobenius norm.  If $B$ is the matrix
representing $S$, then $B^\top B$
represents $S^*S$ and has eigenvalues
$\lambda_1,\ldots,\lambda_m$.  Applying
\eqref{eq:matrix-eckart-young} to $B$ gives
\eqref{eq:finite-domain-eckart-young} and shows that the truncated
singular-value decomposition attains the infimum.
\end{proof}

\subsection{Determinant bound from Kolmogorov widths}
\label{app:width-product-proof}

\begin{proof}[Proof of Lemma~\ref{lem:width-product}]
Fix $\mathbf{x}=(x_1,\ldots,x_m)\in\mathcal D^m$, and let
$S_{\mathbf{x}}:\mathbb R^m\to\mathcal H_k$ be the synthesis operator
$$
  S_{\mathbf{x}}c:=\sum_{\ell=1}^m c_\ell k(\cdot,x_\ell).
$$
Write
$\lambda_1(\mathbf{x})\geq\cdots\geq\lambda_m(\mathbf{x})\geq0$ for the
eigenvalues of
$K_{\mathbf{x}}=S_{\mathbf{x}}^*S_{\mathbf{x}}$.  Given an integer
$0\leq i<m$ and $\varepsilon>0$, choose, by
\eqref{eq:kernel-width}, a subspace $V\subset\mathcal H_k$ with
$\dim V\leq i$ whose uniform approximation error is less than
$d_i(k,\mathcal D)+\varepsilon$, and let
$\Pi_V$ be the orthogonal projector onto~$V$.  Since
$\Pi_VS_{\mathbf{x}}$ has rank at most $i$,
Lemma~\ref{lem:finite-domain-eckart-young} gives
\begin{equation}
  \sum_{\ell>i}\lambda_\ell(\mathbf{x})
  \leq
  \lVert S_{\mathbf{x}}-\Pi_VS_{\mathbf{x}}\rVert_{\mathrm{HS}}^2
  =
  \sum_{\ell=1}^m
  \dist_{\mathcal H_k}\bigl(k(\cdot,x_\ell),V\bigr)^2
  \leq m\{d_i(k,\mathcal D)+\varepsilon\}^2.
  \label{eq:width-eigenvalue-tail}
\end{equation}
Letting $\varepsilon\downarrow0$ and using the nonincreasing order of the
eigenvalues gives, for all integers $0\leq i<j\leq m$,
\begin{equation}
  (j-i)\lambda_j(\mathbf{x})
  \leq\sum_{\ell=i+1}^j\lambda_\ell(\mathbf{x})
  \leq m d_i(k,\mathcal D)^2.
  \label{eq:eigenvalue-width-bridge}
\end{equation}
Moreover, the determinant is the product of these eigenvalues, so
$$
  \det(K_{\mathbf{x}})^{1/(2m)}
  =\biggl(\prod_{j=1}^m\lambda_j(\mathbf{x})\biggr)^{1/(2m)}.
$$
Substitute $i=i_j$ into the preceding eigenvalue bound and take square roots
for each $j$.  Multiplying the resulting inequalities over $j=1,\ldots,m$
and taking the $m$th root gives
$$
  \det(K_{\mathbf{x}})^{1/(2m)}
  \leq
  \biggl\{
    \prod_{j=1}^m
    \sqrt{\frac{m}{j-i_j}}\,d_{i_j}(k,\mathcal D)
  \biggr\}^{1/m}.
$$
Taking the supremum over $\mathbf{x}\in\mathcal D^m$ and using
\eqref{eq:gram-volume-envelope} proves
\eqref{eq:rank-profile-width-product}.  Taking
$i_j=\lfloor j/2\rfloor$ gives
$$
  \lambda_j(\mathbf{x})^{1/2}
  \leq\sqrt{\frac{2m}{j}}\,
       d_{\lfloor j/2\rfloor}(k,\mathcal D).
$$
Multiplying these inequalities and using $m!\geq(m/e)^m$ proves
\eqref{eq:width-product}.
The inequalities remain valid when an eigenvalue or a width is zero.  In
particular, $K_{\mathbf{x}}$ need not be invertible.
\end{proof}

\subsection{From width decay to separation-radius decay}
\label{app:width-envelopes-proof}

\begin{proof}[Proof of Corollary~\ref{cor:width-envelopes}]
Substitute \eqref{eq:polynomial-width-decay} into
\eqref{eq:width-product} for $m\geq2$.  Since
$d_0(k,\mathcal D)\leq1$ and
$\lfloor j/2\rfloor\geq j/3$ for $j\geq2$, we obtain
\begin{equation}
  \delta_m(k,\mathcal D)
  \leq \sqrt{2e}\,
    A^{(m-1)/m}3^\alpha(m!)^{-\alpha/m}
  \leq C_{\mathrm w}m^{-\alpha}.
  \label{eq:polynomial-width-proof-bound}
\end{equation}
The bound $m!\geq(m/e)^m$ allows $C_{\mathrm w}$ to be chosen independently
of $m$.  The case $m=1$ follows from
$\delta_1(k,\mathcal D)=d_0(k,\mathcal D)\leq1$, after increasing
$C_{\mathrm w}$ if necessary.

To prove \eqref{eq:stretched-exponential-separation-decay}, set
$i_j=j-1$ in \eqref{eq:rank-profile-width-product}.  Then
\begin{equation}
  \delta_m(k,\mathcal D)
  \leq
  \sqrt m
  \left\{\prod_{i=0}^{m-1}d_i(k,\mathcal D)\right\}^{1/m}.
  \label{eq:stretched-width-product-bound}
\end{equation}
If $\delta_m(k,\mathcal D)=0$, the inequality in
\eqref{eq:stretched-exponential-separation-decay} holds at that $m$.
Otherwise, the
product in \eqref{eq:stretched-width-product-bound} is positive, so
$d_i(k,\mathcal D)>0$ for every $0\leq i<m$ and logarithms may be taken.
Since $d_0(k,\mathcal D)\leq1$,
\eqref{eq:stretched-exponential-width-decay} gives
\begin{equation}
  \log\delta_m(k,\mathcal D)
  \leq
  \frac12\log m
  +\frac{m-1}{m}\log A
  -\frac{b}{m}\sum_{i=1}^{m-1}i^\gamma\log(ei).
  \label{eq:stretched-width-log-bound}
\end{equation}
Moreover, an integral comparison gives
\begin{equation}
  \sum_{i=1}^{m-1}i^\gamma\log(ei)
  =
  \frac{m^{1+\gamma}\log m}{1+\gamma}
  +O(m^{1+\gamma}),
  \label{eq:stretched-width-sum-expansion}
\end{equation}
and hence
\begin{equation}
  \frac1m\sum_{i=1}^{m-1}i^\gamma\log(ei)
  =
  \left\{\frac1{1+\gamma}
  +O\left(\frac1{\log m}\right)\right\}m^\gamma\log m.
  \label{eq:stretched-width-average-asymptotic}
\end{equation}
Let $0<b'<b/(1+\gamma)$, and choose
$\widetilde b$ with $b'<\widetilde b<b/(1+\gamma)$.
Substituting \eqref{eq:stretched-width-average-asymptotic} into
\eqref{eq:stretched-width-log-bound} gives, for all sufficiently large~$m$
such that
$\delta_m(k,\mathcal D)>0$,
$$
  \log\delta_m(k,\mathcal D)
  \leq-\widetilde b\,m^\gamma\log m.
$$
For all sufficiently large~$m$,
$\widetilde b\log m\geq b'\log(em)$.  Hence
$\delta_m(k,\mathcal D)\leq\exp\{-b'm^\gamma\log(em)\}$ for those $m$.
Since
$\delta_m(k,\mathcal D)\leq
\delta_1(k,\mathcal D)=d_0(k,\mathcal D)\leq1$,
increasing $A'$ if necessary extends
\eqref{eq:stretched-exponential-separation-decay} to the finitely many
smaller values of $m$.
\end{proof}

\section{Separation-radius decay and regret rates}
\label{app:kernel-estimates}

This appendix proves Propositions~\ref{prop:polynomial-transfer}
and~\ref{prop:log-enhanced-transfer} and the Mat\'ern and
squared-exponential separation-radius estimates used in
Section~\ref{sec:kernels}.  The bounds valid in every dimension use
Cartesian-grid approximations.  The sharper one-dimensional
squared-exponential argument instead combines the conditional-variance
estimate proved by \citet[Theorem~2]{yarotsky2013} with
\eqref{eq:optimizer-uncertainty-recursion}.

\subsection{Regret from polynomial decay of the separation radii}
\label{app:polynomial-transfer-proof}

\begin{proof}[Proof of Proposition~\ref{prop:polynomial-transfer}]
By \eqref{eq:regret-monotonicity},
\eqref{eq:selected-point-innovation}, and the one-step regret bound
\eqref{eq:one-step}, the simple-regret sequence and the selected-point
innovation norms satisfy the hypotheses of
Theorem~\ref{thm:order-statistic-transfer}.  If $q=0$, take $L=1$, so
$M=N$.  Theorem~\ref{thm:order-statistic-transfer} then gives
$r_T\leq C\delta_N(k_0,\X)$, and the result follows
from~\eqref{eq:polynomial-envelope}.  Suppose $0<q<1$.  With
$N=T-n_0$, choose
$$
  L:=
  \left\lceil
    \frac{\alpha+1}{-\log q}\log(eN)
  \right\rceil,
  \qquad
  M:=N-L+1.
$$
For all sufficiently large~$N$, this choice satisfies $1\leq L\leq N$.
Moreover, $L=O(\log N)$,~$M=N-O(\log N)$, and
$q^L\leq(eN)^{-\alpha-1}$.  Theorem~\ref{thm:order-statistic-transfer}
and~\eqref{eq:polynomial-envelope} give
$$
  r_T
  \leq q^Lr_{n_0}
  +\frac{CA}{1-q}M^{-\alpha}\{\log(eM)\}^{\beta}.
$$
The bound $q^L\leq(eN)^{-\alpha-1}$ gives
$q^Lr_{n_0}=O(N^{-\alpha-1})$.  Since $M=N-L+1=N-O(\log N)$, we have
$M\sim N$, and therefore
$$
  M^{-\alpha}\{\log(eM)\}^{\beta}
  =O\bigl(N^{-\alpha}(\log N)^\beta\bigr).
$$
Combining these estimates,
$$
  r_T=O\bigl(N^{-\alpha}(\log N)^\beta\bigr),
$$
which proves~\eqref{eq:polynomial-rate}.
For every objective function $f$ satisfying~\eqref{eq:rkhs-ball},
Remark~\ref{rem:ei-constants} and~\eqref{eq:selected-point-innovation} give
$$
  r_{n+1}\leq q_Br_n+C_Bv_{n+1},
$$
where $0\leq q_B<1$ and $C_B>0$ are independent of $f$.
Moreover, the reproducing property, \eqref{eq:normalization},
and~\eqref{eq:rkhs-ball} give $r_{n_0}\leq2B$.
The residual kernel $k_0$ depends only on the kernel and the initial
query points.  Thus \eqref{eq:polynomial-envelope}, including its starting rank, is
independent of $f$.  If $q_B=0$, the preceding argument applies with $L=1$.
If $0<q_B<1$, use
$$
  L=
  \left\lceil
    \frac{\alpha+1}{-\log q_B}\log(eN)
  \right\rceil.
$$
Together with $r_{n_0}\leq2B$, these choices give
\eqref{eq:polynomial-rate} with an implied constant and a starting index
that do not depend on the choice of $f\in\Hk$ with
$\lVert f\rVert_{\Hk}\leq B$.
\end{proof}

\subsection{Mat\'ern separation-radius decay}
\label{app:matern-separation-radius}

\begin{lemma}[Mat\'ern power functions and widths]
\label{lem:matern-widths}
Let $k_\nu$ be a stationary isotropic Mat\'ern kernel on $\R^d$, normalized
to have unit diagonal, with fixed lengthscale and smoothness~$\nu>0$.
Let $\mathcal Q\subset\R^d$ be a bounded open cube and set
$k_{\nu,\mathcal Q}:=k_\nu|_{\mathcal Q\times\mathcal Q}$.  If
$\mathcal Y\subset\mathcal Q$ is finite and nonempty, define its fill
distance and power function by
\begin{equation}
  h_{\mathcal Y,\mathcal Q}
  :=\sup_{x\in\mathcal Q}\min_{y\in\mathcal Y}\lVert x-y\rVert,
  \qquad
  P_{\mathcal Y}^{\mathcal Q}(x):=
  \dist_{\mathcal H_{k_{\nu,\mathcal Q}}}
  \left(
    k_{\nu,\mathcal Q}(\cdot,x),
    \operatorname{span}
    \{k_{\nu,\mathcal Q}(\cdot,y),\,y\in\mathcal Y\}
  \right).
  \label{eq:matern-cube-power-definitions}
\end{equation}
There are constants $C_{\mathcal Q},h_0>0$ such that, for every finite
nonempty $\mathcal Y\subset\mathcal Q$,
\begin{equation}
  \sup_{x\in\mathcal Q}P_{\mathcal Y}^{\mathcal Q}(x)
  \leq C_{\mathcal Q} h_{\mathcal Y,\mathcal Q}^{\nu}
  \qquad\text{whenever }h_{\mathcal Y,\mathcal Q}\leq h_0.
  \label{eq:matern-cube-power}
\end{equation}
Consequently, for every nonempty compact $\X\subset\R^d$, there is
$A_\X<\infty$ such that
\begin{equation}
  d_m(k_\nu,\X)\leq A_\X m^{-\nu/d},
  \qquad m\geq1.
  \label{eq:matern-width-appendix}
\end{equation}
The constant $A_\X$ depends only on the kernel parameters and on the cube
chosen to contain~$\X$.  No regularity of $\partial\X$ is required.
\end{lemma}

\begin{proof}
Write $k_\nu(x,y)=\Phi_\nu(x-y)$.  There are constants $c_0,c_1>0$,
depending only on the kernel parameters, such that the Mat\'ern spectral
density is
\begin{equation}
  \widehat\Phi_\nu(\omega)
  =
  c_0
  \bigl(c_1+\lVert\omega\rVert^2\bigr)^{-(\nu+d/2)}.
  \label{eq:matern-spectral-density}
\end{equation}
The spectral density is everywhere positive, so $\Phi_\nu$ is strictly
positive definite \citep[Corollary~6.9]{wendland2005}.  As
$\lVert\omega\rVert\to\infty$,
\begin{equation}
  \widehat\Phi_\nu(\omega)
  \asymp
  \lVert\omega\rVert^{-d-2\nu}.
  \label{eq:matern-spectral-tail}
\end{equation}
To bound $P_{\mathcal Y}^{\mathcal Q}$, we apply the power-function estimate
in \citet[Theorem~3.2]{schabackWendland2002}.  The theorem assumes
$\widehat\Phi(\omega)\asymp
\lVert\omega\rVert^{-d-s_\infty}$, and comparison with
\eqref{eq:matern-spectral-tail} gives $s_\infty=2\nu$.  We use the theorem
with $m=0$, where $m$ denotes the order of conditional positive
definiteness.  For $m=0$, no polynomial moment conditions are imposed, and
the condition reduces to strict positive definiteness, established above for
$\Phi_\nu$.  A bounded open cube satisfies the required interior cone
condition.  The power function in the theorem equals
the RKHS residual norm in \eqref{eq:matern-cube-power-definitions}.  Hence
the theorem gives
$$
  P_{\mathcal Y}^{\mathcal Q}(x)
  \leq C_{\mathcal Q} h_{\mathcal Y,\mathcal Q}^{s_\infty/2}
  =C_{\mathcal Q} h_{\mathcal Y,\mathcal Q}^{\nu}
$$
for every $x\in\mathcal Q$ once $h_{\mathcal Y,\mathcal Q}\leq h_0$, proving
\eqref{eq:matern-cube-power}.

Now fix a nonempty compact $\X\subset\R^d$ and choose a bounded open cube
$\mathcal Q$ of side length~$a_{\mathcal Q}>0$ containing $\X$.  For each
integer $n\geq1$, partition $\mathcal Q$ into $n^d$ congruent cubes and let
$\mathcal Y_n$ be their centers.
Then
\begin{equation}
  |\mathcal Y_n|=n^d,
  \qquad
  h_{\mathcal Y_n,\mathcal Q}\leq\frac{a_{\mathcal Q}\sqrt d}{2n}.
  \label{eq:matern-grid}
\end{equation}
Given $m\geq1$, put $n_m:=\lfloor m^{1/d}\rfloor$ and
$\mathcal Y'_m:=\mathcal Y_{n_m}$.  Since $n_m^d\leq m$ and
$n_m\geq m^{1/d}/2$, the space
$$
  V_m^{\mathcal Q}:=
  \operatorname{span}
  \{k_{\nu,\mathcal Q}(\cdot,y),\,y\in\mathcal Y'_m\}
$$
has dimension at most $m$, while
$h_{\mathcal Y'_m,\mathcal Q}\leq
a_{\mathcal Q}\sqrt d\,m^{-1/d}$.  Hence, for all sufficiently
large~$m$,~\eqref{eq:matern-cube-power}, applied with
$\mathcal Y=\mathcal Y'_m$, gives
\begin{equation}
  \sup_{x\in\mathcal Q}
  \dist_{\mathcal H_{k_{\nu,\mathcal Q}}}
  \bigl(k_{\nu,\mathcal Q}(\cdot,x),V_m^{\mathcal Q}\bigr)
  \leq C_{\mathcal Q}(a_{\mathcal Q}\sqrt d)^\nu m^{-\nu/d}.
  \label{eq:matern-cube-width}
\end{equation}

Let $k_{\nu,\X}:=k_\nu|_{\X\times\X}$, and let
$$
  R_{\mathcal Q,\X}:\mathcal H_{k_{\nu,\mathcal Q}}
  \longrightarrow\mathcal H_{k_{\nu,\X}},
  \qquad
  R_{\mathcal Q,\X}g:=g|_\X.
$$
By the restriction theorem
\citep[Section~5]{aronszajn1950},
$$
  \lVert R_{\mathcal Q,\X}g\rVert_{\mathcal H_{k_{\nu,\X}}}
  \leq
  \lVert g\rVert_{\mathcal H_{k_{\nu,\mathcal Q}}},
  \qquad g\in\mathcal H_{k_{\nu,\mathcal Q}}.
$$
Moreover,
$R_{\mathcal Q,\X}k_{\nu,\mathcal Q}(\cdot,x)=k_{\nu,\X}(\cdot,x)$
for $x\in\X$.
The space $W_m:=R_{\mathcal Q,\X}V_m^{\mathcal Q}$ has dimension at most
$m$.  The grid points in $\mathcal Y'_m$ need not belong to $\X$.  Their
restricted kernel
translates belong to $\mathcal H_{k_{\nu,\X}}$ and span $W_m$, which is
admissible in \eqref{eq:kernel-width}.  The interior cone condition is
imposed on the auxiliary cube $\mathcal Q$, while $\X$ may have empty
interior or an irregular boundary.

For every $x\in\X$,
\begin{align*}
  \dist_{\mathcal H_{k_{\nu,\X}}}
  \bigl(k_{\nu,\X}(\cdot,x),W_m\bigr)
  &=
  \inf_{v\in V_m^{\mathcal Q}}
  \left\lVert
    R_{\mathcal Q,\X}\bigl(k_{\nu,\mathcal Q}(\cdot,x)-v\bigr)
  \right\rVert_{\mathcal H_{k_{\nu,\X}}}\\
  &\leq
  \inf_{v\in V_m^{\mathcal Q}}
  \left\lVert k_{\nu,\mathcal Q}(\cdot,x)-v
  \right\rVert_{\mathcal H_{k_{\nu,\mathcal Q}}}\\
  &=
  \dist_{\mathcal H_{k_{\nu,\mathcal Q}}}
  \bigl(k_{\nu,\mathcal Q}(\cdot,x),V_m^{\mathcal Q}\bigr).
\end{align*}

Taking the supremum over $x\in\X$ and using
\eqref{eq:matern-cube-width} proves
\eqref{eq:matern-width-appendix} for all sufficiently large~$m$.
Since $k_\nu(x,x)=1$, one has
$d_m(k_\nu,\X)\leq d_0(k_\nu,\X)=1$.  Increasing $A_\X$ if necessary gives
\eqref{eq:matern-width-appendix} for the finitely many smaller ranks.

\end{proof}

\begin{proof}[Proof of Proposition~\ref{prop:matern-separation-radius}]
Lemma~\ref{lem:matern-widths} gives
\eqref{eq:matern-width-appendix}.  Since $k_\nu(x,x)=1$,
the polynomial case of Corollary~\ref{cor:width-envelopes}, with
$\alpha=\nu/d$, then yields~\eqref{eq:matern-innovation}.
\end{proof}

\subsection{Simple-regret rate for EI with a Mat\'ern kernel}
\label{app:matern-rate-proof}

\begin{proof}[Proof of Theorem~\ref{thm:matern-rate}]
By Lemma~\ref{lem:residual-envelope-domination},
$\delta_m(k_0,\X)\leq\delta_m(k_\nu,\X)$, so
$\delta_m(k_0,\X)=O(m^{-\nu/d})$.
Proposition~\ref{prop:polynomial-transfer}, with $\alpha=\nu/d$, $\beta=0$,
and $q_B,C_B$, then gives~\eqref{eq:matern-rate}.  The constants $q_B,C_B$
depend on $B,\sigma,\eta$.  The implied constant and the rank from which
\eqref{eq:matern-innovation} holds depend only on the fixed kernel and domain.
Since $r_{n_0}\leq2B$, the implied constant and starting index
in~\eqref{eq:matern-rate} do not depend on the choice of
$f\in\Hk$ with $\lVert f\rVert_{\Hk}\leq B$.
\end{proof}

\subsection{Regret from exponential bounds on the separation radii}
\label{app:log-enhanced-transfer-proof}

\begin{lemma}[Choice of $L$ in the finite-budget bound]
\label{lem:log-enhanced-choice}
Assume~\eqref{eq:log-enhanced-envelope} and fix $q\in[0,1)$.
Recall from Proposition~\ref{prop:log-enhanced-transfer} that
\begin{equation}
  \psi_d(N)=\min\{N,N^{1/d}\log(eN)\}.
  \label{eq:log-enhanced-scale}
\end{equation}
There are constants $\widetilde A,\widetilde b>0$ and an integer
$N_1\geq1$ such that, for every integer $N\geq N_1$, the choices
\begin{equation}
  L=
  \begin{cases}
    1, & q=0,\\
    \lfloor N/2\rfloor, & 0<q<1,\ d=1,\\
    \lceil N^{1/d}\log(eN)\rceil, & 0<q<1,\ d\geq2,
  \end{cases}
  \qquad
  M:=N-L+1,
  \label{eq:log-enhanced-indices}
\end{equation}
satisfy $1\leq L\leq N$ and
\begin{equation}
  q^L+\delta_M(k_0,\X)
  \leq
  \widetilde A\exp\{-\widetilde b\psi_d(N)\}.
  \label{eq:log-enhanced-index-choice}
\end{equation}
The constants and $N_1$ depend only on $A,b,d,q$ and the rank
from which~\eqref{eq:log-enhanced-envelope} holds.
\end{lemma}

\begin{proof}
If $q=0$, \eqref{eq:log-enhanced-indices} gives $L=1$ and $M=N$.
Thus $q^L=0$, and the left-hand side of
\eqref{eq:log-enhanced-index-choice} is
$\delta_N(k_0,\X)$.  When $d=1$,
\eqref{eq:log-enhanced-scale} gives $\psi_1(N)=N$, and, for all sufficiently
large $N$,
$$
  \delta_N(k_0,\X)
  \leq A\exp\{-bN\log N\}
  \leq A\exp\{-bN\}
  =A\exp\{-b\psi_1(N)\}.
$$
When $d\geq2$, for all sufficiently large $N$,
\eqref{eq:log-enhanced-scale} gives
$\psi_d(N)=N^{1/d}\log(eN)$, and
$\log N\geq\frac12\log(eN)$.  Hence
$$
  \delta_N(k_0,\X)
  \leq A\exp\{-bN^{1/d}\log N\}
  \leq A\exp\{-\tfrac b2\psi_d(N)\}.
$$
This proves~\eqref{eq:log-enhanced-index-choice} when $q=0$.

Suppose $0<q<1$.  When $d=1$, \eqref{eq:log-enhanced-indices} gives
$L=\lfloor N/2\rfloor$ and $M=N-L+1$.  For all sufficiently large $N$,
$$
  L\geq\frac N3,
  \qquad M\geq\frac N2,
  \qquad \log M\geq\frac12\log(eN),
$$
and hence
$$
  q^L\leq e^{-|\log q|N/3},
  \qquad
  \delta_M(k_0,\X)
  \leq A\exp\{-\tfrac b4N\log(eN)\}.
$$
Since $\psi_1(N)=N$ and $\log(eN)\geq1$, setting
$$
  c_1:=\min\{|\log q|/3,b/4\}>0
$$
gives
$$
  q^L+\delta_M(k_0,\X)
  \leq(1+A)\exp\{-c_1\psi_1(N)\}.
$$

When $d\geq2$, \eqref{eq:log-enhanced-indices} gives
$L=\lceil N^{1/d}\log(eN)\rceil$ and $M=N-L+1$.
For all sufficiently large $N$, this choice satisfies $1\leq L\leq N$.
Moreover, $M\geq N/2$ and
$\log M\geq\frac12\log(eN)$.  Therefore
$$
  q^L\leq\exp\{-|\log q|N^{1/d}\log(eN)\},
  \qquad
  M^{1/d}\log M
  \geq2^{-1-1/d}N^{1/d}\log(eN).
$$
By~\eqref{eq:log-enhanced-envelope} and the identity
$\psi_d(N)=N^{1/d}\log(eN)$, setting
$$
  c_2:=\min\{|\log q|,b\,2^{-1-1/d}\}>0
$$
gives
$$
  q^L+\delta_M(k_0,\X)
  \leq(1+A)\exp\{-c_2\psi_d(N)\}.
$$
These estimates prove
\eqref{eq:log-enhanced-index-choice}.
\end{proof}

\begin{proof}[Proof of Proposition~\ref{prop:log-enhanced-transfer}]
By \eqref{eq:regret-monotonicity}, \eqref{eq:selected-point-innovation}, and
the one-step regret bound \eqref{eq:one-step}, the simple-regret sequence and
selected-point innovation norms satisfy the hypotheses of
Theorem~\ref{thm:order-statistic-transfer}.
For the values of $L$ and $M$ in~\eqref{eq:log-enhanced-indices},
\eqref{eq:order-transfer-coarse} gives
$$
  \begin{aligned}
  r_T
  &\leq
  r_{n_0}q^L+\frac{C}{1-q}\delta_M(k_0,\X)\\
  &\leq
  \left(r_{n_0}+\frac{C}{1-q}\right)
  \widetilde A\exp\{-\widetilde b\psi_d(N)\}.
  \end{aligned}
$$
This proves~\eqref{eq:log-enhanced-rate}.

Remark~\ref{rem:ei-constants} gives
$q\leq q_B<1$ and $C\leq C_B$ for every $f\in\Hk$ with
$\lVert f\rVert_{\Hk}\leq B$.
Since $r_n\geq0$ and $s_n(x_{n+1})\geq0$, the one-step regret bound
remains valid with $q_B,C_B$ in place of $q,C$.  Moreover,
\eqref{eq:normalization} and~\eqref{eq:rkhs-ball} give
$r_{n_0}\leq2B$.  For the fixed initial design,
\eqref{eq:log-enhanced-envelope}, including its starting rank, is independent
of $f$.  Applying Lemma~\ref{lem:log-enhanced-choice} with $q=q_B$ and
\eqref{eq:order-transfer-coarse} with $q_B,C_B$ gives constants and a
starting index independent of $f$.
\end{proof}

\subsection{Squared-exponential separation-radius decay}
\label{app:se-separation-radius}

\begin{lemma}[Squared-exponential power function on a cube]
\label{lem:se-cube-power}
Let $\mathcal Q\subset\R^d$ be a closed cube of positive side length, fix
$\varrho>0$,
and let
$$
  k_{\mathcal Q}(x,y):=
  \exp\biggl(-\frac{\lVert x-y\rVert^2}{2\varrho^2}\biggr),
  \qquad x,y\in\mathcal Q.
$$
Define, for every finite nonempty $\mathcal Y\subset\mathcal Q$,
\begin{equation}
  h_{\mathcal Y,\mathcal Q}
  :=\sup_{x\in\mathcal Q}\min_{y\in\mathcal Y}\lVert x-y\rVert,
  \qquad
  V_{\mathcal Y}^{\mathcal Q}
  :=\operatorname{span}\{k_{\mathcal Q}(\cdot,y),\,y\in\mathcal Y\}.
  \label{eq:se-cube-power-definitions}
\end{equation}
There are constants $c_0>0$ and $h_0\in(0,1)$, depending only on
$d,\varrho,\mathcal Q$, such that
\begin{equation}
  \sup_{x\in\mathcal Q}
  \dist_{\mathcal H_{k_{\mathcal Q}}}
  \bigl(k_{\mathcal Q}(\cdot,x),V_{\mathcal Y}^{\mathcal Q}\bigr)
  \leq
  \exp\left\{-c_0
    \frac{|\log h_{\mathcal Y,\mathcal Q}|}
         {h_{\mathcal Y,\mathcal Q}}\right\},
  \label{eq:se-cube-power}
\end{equation}
whenever $h_{\mathcal Y,\mathcal Q}\leq h_0$.
\end{lemma}

\begin{proof}
The spectral density of the squared-exponential kernel on $\R^d$ is positive
everywhere, so its restriction $k_{\mathcal Q}$ is strictly positive
definite \citep[Corollary~6.9]{wendland2005}.  Hence, for every
$u\in\mathcal H_{k_{\mathcal Q}}$, there is a unique
$I_{\mathcal Y}u\in V_{\mathcal Y}^{\mathcal Q}$ that agrees with
$u$ on $\mathcal Y$.

To apply \citet[Theorem~11.22]{wendland2005} to the squared-exponential
kernel, take the auxiliary function
$$
  g(t):=e^{-t/(2\varrho^2)}
  \quad(t\geq0),
$$
because $k_{\mathcal Q}(x,y)=g(\lVert x-y\rVert^2)$.
For every integer $j\geq0$,
$$
  |g^{(j)}(t)|
  =(2\varrho^2)^{-j}e^{-t/(2\varrho^2)}
  \leq(2\varrho^2)^{-j}
  \qquad(t\geq0).
$$
Thus the derivative hypothesis in Theorem~11.22 is satisfied.  The theorem
and Equation~(11.11) give constants $c_0>0$ and $h_0\in(0,1)$, depending
only on $d,\varrho,\mathcal Q$, for which every
$u\in\mathcal H_{k_{\mathcal Q}}$ satisfies
\begin{equation}
  \sup_{x\in\mathcal Q}|u(x)-I_{\mathcal Y}u(x)|
  \leq
  \exp\left\{-c_0
    \frac{|\log h_{\mathcal Y,\mathcal Q}|}
         {h_{\mathcal Y,\mathcal Q}}\right\}
  \lVert u\rVert_{\mathcal H_{k_{\mathcal Q}}},
  \label{eq:se-wendland-error}
\end{equation}
whenever $h_{\mathcal Y,\mathcal Q}\leq h_0$.
For $x\in\mathcal Q$, the power-function identity gives
\begin{equation}
  \dist_{\mathcal H_{k_{\mathcal Q}}}
  \bigl(k_{\mathcal Q}(\cdot,x),V_{\mathcal Y}^{\mathcal Q}\bigr)
  =
  \sup_{\lVert u\rVert_{\mathcal H_{k_{\mathcal Q}}}\leq1}
  |u(x)-I_{\mathcal Y}u(x)|.
  \label{eq:se-power-function-identity}
\end{equation}
Taking the supremum over $x\in\mathcal Q$ in
\eqref{eq:se-power-function-identity} and applying
\eqref{eq:se-wendland-error} proves~\eqref{eq:se-cube-power}.
\end{proof}

\begin{lemma}[Squared-exponential Kolmogorov widths]
\label{lem:se-widths}
Let $\X\subset\R^d$ be nonempty and compact, and fix $\varrho>0$.  Let
$k:\X\times\X\to\R$ be the restriction of the squared-exponential kernel
$$
  (x,y)\longmapsto
  \exp\left(-\frac{\lVert x-y\rVert^2}{2\varrho^2}\right).
$$
There are constants $A_{\mathrm w},b_{\mathrm w}>0$ such that
\begin{equation}
  d_m(k,\X)
  \leq
  A_{\mathrm w}
  \exp\{-b_{\mathrm w}m^{1/d}\log(em)\},
  \qquad m\geq1.
  \label{eq:se-width-appendix}
\end{equation}
The constants depend only on $d$, $\varrho$, and the cube chosen to contain
$\X$.
\end{lemma}

\begin{proof}
Choose a closed cube $\mathcal Q\subset\R^d$ of side length
$a_{\mathcal Q}>0$ containing $\X$, and define
$$
  k_{\mathcal Q}(x,y):=
  \exp\biggl(-\frac{\lVert x-y\rVert^2}{2\varrho^2}\biggr),
  \qquad x,y\in\mathcal Q.
$$
Let $c_0>0$ and $h_0\in(0,1)$ be the constants given by
Lemma~\ref{lem:se-cube-power} for this cube.
Given $m\geq1$, write $n_m:=\lfloor m^{1/d}\rfloor$ and denote by
$\mathcal Y_m$ the centers of the $n_m^d$ congruent subcubes obtained by
partitioning $\mathcal Q$.  Define
$$
  V_{\mathcal Y_m}^{\mathcal Q}
  :=
  \operatorname{span}\{k_{\mathcal Q}(\cdot,y),\,y\in\mathcal Y_m\}.
$$
Then $|\mathcal Y_m|=n_m^d\leq m$, $n_m\geq m^{1/d}/2$, and the fill
distance of $\mathcal Y_m$ in $\mathcal Q$ satisfies
$$
  h_{\mathcal Y_m,\mathcal Q}
  \leq\frac{a_{\mathcal Q}\sqrt d}{2n_m}.
$$
The map $h\mapsto|\log h|/h$ is decreasing on $(0,1)$.  Hence the bounds on
$n_m$ and $h_{\mathcal Y_m,\mathcal Q}$ imply that there is $b_1>0$,
depending only on $d,\varrho,\mathcal Q$, such that
\begin{equation}
  h_{\mathcal Y_m,\mathcal Q}\leq h_0,
  \qquad
  c_0\frac{|\log h_{\mathcal Y_m,\mathcal Q}|}
           {h_{\mathcal Y_m,\mathcal Q}}
  \geq b_1m^{1/d}\log(em),
  \label{eq:se-grid-exponent}
\end{equation}
for all sufficiently large $m$.  Since
$\dim V_{\mathcal Y_m}^{\mathcal Q}\leq m$, applying
Lemma~\ref{lem:se-cube-power} with
$\mathcal Y=\mathcal Y_m$ and using~\eqref{eq:se-grid-exponent} gives
\begin{equation}
  d_m(k_{\mathcal Q},\mathcal Q)
  \leq
  \sup_{x\in\mathcal Q}
  \dist_{\mathcal H_{k_{\mathcal Q}}}
  \bigl(k_{\mathcal Q}(\cdot,x),V_{\mathcal Y_m}^{\mathcal Q}\bigr)
  \leq
  \exp\{-b_1m^{1/d}\log(em)\}.
  \label{eq:se-cube-width}
\end{equation}

As in the proof of Lemma~\ref{lem:matern-widths}, restriction from
$\mathcal Q$ to $\X$ is contractive and does not increase subspace
dimensions.  Therefore,
$$
  d_m(k,\X)\leq d_m(k_{\mathcal Q},\mathcal Q).
$$
Thus~\eqref{eq:se-width-appendix} holds with $b_{\mathrm w}=b_1$ for all
sufficiently large $m$.
Since $k(x,x)=1$, one has $d_m(k,\X)\leq d_0(k,\X)=1$.  Increasing
$A_{\mathrm w}$ if necessary gives~\eqref{eq:se-width-appendix} for the
finitely many smaller ranks.
\end{proof}

\begin{proposition}[One-dimensional power-function bound]
\label{prop:yarotsky-power}
Let $\X\subset\R$ be nonempty and compact, and fix $\varrho>0$.  Let
$k:\X\times\X\to\R$ be given by
$$
  k(x,y)=\exp\left(-\frac{(x-y)^2}{2\varrho^2}\right).
$$
For every $0<b<1/2$, there is a constant $A_b\geq1$, depending only on
$b$, $\varrho$, and $\X$, such that, for every integer $K\geq1$ and every
collection of pairwise distinct points $x_1,\ldots,x_K\in\X$,
\begin{equation}
  \sup_{x\in\X}P_{(x_1,\ldots,x_K)}(x)
  \leq A_b\exp\{-bK\log(eK)\}.
  \label{eq:yarotsky-power-bound}
\end{equation}
Moreover, with $\log 0=-\infty$,
$$
  \limsup_{m\to\infty}
  \frac{\log\delta_m(k,\X)}{m\log m}
  \leq-\frac12.
$$
\end{proposition}

\begin{proof}
Choose $c_{\X}\in\R$ and $a_{\X}>0$ such that
$\X\subset[c_{\X}-a_{\X},c_{\X}+a_{\X}]$, and set
\begin{equation}
  \widetilde x:=\frac{x-c_{\X}}{a_{\X}},\qquad
  \widetilde x_i:=\frac{x_i-c_{\X}}{a_{\X}},\qquad
  \gamma:=\frac{\varrho^2}{2a_{\X}^2}.
  \label{eq:yarotsky-rescaling}
\end{equation}
All the rescaled points lie in $[-1,1]$.  With the Fourier convention of
\citet{yarotsky2013}, Equation~(16) with $a=\gamma$ and spectral exponent
$2$ gives the spectral density
$$
  \widehat g_\gamma(t):=\exp(-\gamma t^2).
$$
Its inverse Fourier transform is
$$
  g_\gamma(h)
  =\int_{\R}\exp(-\gamma t^2)e^{ith}\,dt
  =\sqrt{\frac{\pi}{\gamma}}\,
    \exp\left(-\frac{h^2}{4\gamma}\right),
  \qquad h\in\R.
$$
Since $4\gamma a_{\X}^2=2\varrho^2$, the rescaling in
\eqref{eq:yarotsky-rescaling} gives
$$
  k(x,y)=\sqrt{\frac{\gamma}{\pi}}\,
  g_\gamma\left(
    \frac{x-c_{\X}}{a_{\X}}-\frac{y-c_{\X}}{a_{\X}}
  \right).
$$
Given a finite tuple $\mathbf y$ of real numbers, write
$P_{\mathbf y}^{g_\gamma}$ for the power function associated with the kernel
$(s,t)\mapsto g_\gamma(s-t)$.  Its square is the conditional variance of a
centered Gaussian process with this covariance function after conditioning
at the points in $\mathbf y$.
Conditional variances scale linearly with the covariance function.  Thus
\begin{equation}
  \bigl(P_{(x_1,\ldots,x_K)}(x)\bigr)^2
  =\sqrt{\frac{\gamma}{\pi}}\,
    \bigl(
      P_{(\widetilde x_1,\ldots,\widetilde x_K)}^{g_\gamma}
      (\widetilde x)
    \bigr)^2.
  \label{eq:yarotsky-kernel-scaling}
\end{equation}

For all sufficiently large $K$ and all distinct points
$\widetilde x,\widetilde x_1,\ldots,\widetilde x_K$,
\citet[Theorem~2 and Equation~(14)]{yarotsky2013}, applied to
$\widehat g_\gamma(t)=\exp(-\gamma t^2)$, gives
\begin{equation}
  \bigl(
    P_{(\widetilde x_1,\ldots,\widetilde x_K)}^{g_\gamma}
    (\widetilde x)
  \bigr)^2
  \leq
  \exp\{F_\gamma(K)+2K\}
  \prod_{i=1}^K|\widetilde x-\widetilde x_i|^2,
  \label{eq:yarotsky-conditional-variance}
\end{equation}
where
\begin{equation}
  F_\gamma(K)
  :=
  \frac{2K+1}{2}
  \left\{
    \log\left(\frac{2K+1}{2\gamma}\right)
    -2\log K-1
  \right\}
  =-K\log K+O_\gamma(K).
  \label{eq:yarotsky-f-asymptotic}
\end{equation}
Here $O_\gamma(K)$ denotes a quantity whose absolute value is bounded by
$C_\gamma K$, where $C_\gamma$ may depend on $\gamma$ but not on $K$.
Since $|\widetilde x-\widetilde x_i|\leq2$, using
\eqref{eq:yarotsky-f-asymptotic} in
\eqref{eq:yarotsky-kernel-scaling} and
\eqref{eq:yarotsky-conditional-variance} gives
\begin{equation}
  \bigl(P_{(x_1,\ldots,x_K)}(x)\bigr)^2
  \leq
  \sqrt{\frac{\gamma}{\pi}}\,
  \exp\{F_\gamma(K)+2K+K\log4\}
  =\exp\{-K\log K+O_\gamma(K)\}.
  \label{eq:yarotsky-power-asymptotic}
\end{equation}
If $x$ is one of the points $x_i$, then the power function is zero.  The
bound~\eqref{eq:yarotsky-power-asymptotic} therefore holds for every
$x\in\X$.  Taking square roots shows that, for every $0<b<1/2$,
\eqref{eq:yarotsky-power-bound} holds with $A_b=1$ for all sufficiently
large $K$.  Since
$P_{(x_1,\ldots,x_K)}(x)\leq\sqrt{k(x,x)}=1$, choosing $A_b\geq1$ large
enough also covers the finitely many smaller values of $K$.

Now let $(x_1,\ldots,x_m)\in\X^m$.  If the tuple contains a repeated point,
the innovation norm at the first repetition is zero.  If the points are
distinct and $m\geq2$, the minimum of the successive innovation norms is at
most the final one, and~\eqref{eq:yarotsky-power-bound} gives
\begin{equation}
  \min_{1\leq j\leq m}
  P_{(x_1,\ldots,x_{j-1})}(x_j)
  \leq
  A_b\exp\{-b(m-1)\log(e(m-1))\}.
  \label{eq:yarotsky-innovation-bound}
\end{equation}
Taking the supremum over ordered tuples and including the repeated-point case
gives, for $m\geq2$,
$$
  \delta_m(k,\X)
  \leq
  A_b\exp\{-b(m-1)\log(e(m-1))\}.
$$
Dividing logarithms by $m\log m$ and then letting $b\uparrow1/2$ proves the
second assertion.
\end{proof}

\begin{proof}[Proof of Proposition~\ref{prop:se-separation-radius}]
Lemma~\ref{lem:se-widths} gives~\eqref{eq:se-width-appendix}.  Since
$k(x,x)=1$, part~2 of
Corollary~\ref{cor:width-envelopes}, with $\gamma=1/d$, shows that, for every
$0<b_\delta<b_{\mathrm w}/(1+1/d)$, there is $A_\delta>0$ such that
$$
  \delta_m(k,\X)
  \leq
  A_\delta
  \exp\{-b_\delta m^{1/d}\log(em)\},
  \qquad m\geq1.
$$
Here $A_\delta$ depends only on $d$, $\varrho$, the cube chosen to contain
$\X$, and the chosen $b_\delta$.  Choosing any admissible $b_\delta$
proves~\eqref{eq:se-innovation}.  When $d=1$, the
second assertion of Proposition~\ref{prop:yarotsky-power} gives
\eqref{eq:se-one-dimensional-leading-constant}.
\end{proof}

\subsection{Simple-regret rates for EI with a squared-exponential kernel}
\label{app:se-rate-proof}

\begin{proof}[Proof for weak-EI policies]
By Lemma~\ref{lem:residual-envelope-domination},
$$
  \delta_m(k_0,\X)\leq\delta_m(k,\X).
$$
Thus \eqref{eq:log-enhanced-envelope} holds for $k_0$, with constants and a
threshold rank independent of $f$.
Proposition~\ref{prop:log-enhanced-transfer} then gives
\eqref{eq:se-rate}.  The same constants and starting index apply to every
$f\in\Hk$ with $\lVert f\rVert_{\Hk}\leq B$.
\end{proof}

\begin{proof}[Proof for exact~EI when $d=1$]
The Euclidean spectral density of the squared-exponential kernel is positive
everywhere, so the kernel is strictly positive definite
\citep[Corollary~6.9]{wendland2005}.  At a previously queried point $y$,
interpolation gives
$s_n(y)=0$ and
$\mu_n(y)=f(y)$, so $\EI_n(y)=0$.  At an unqueried point $x$, strict
positive definiteness gives $s_n(x)>0$, and~\eqref{eq:ei-closed-form} gives
$\EI_n(x)>0$.  The exact~EI policy therefore selects an unqueried point
whenever one exists.  If every point of $\X$ has been queried by time $T-1$,
then $r_T=0$, so \eqref{eq:se-one-dimensional-exact-rate} holds.  Otherwise,
every post-initial point selected before $x_T$ is new.

Let
$$
  D_0:=\bigl|\{x_1,\ldots,x_{n_0}\}\bigr|,
  \qquad
  S_{T-1}:=\{x_1,\ldots,x_{T-1}\}.
$$
Since the $N-1$ post-initial points selected before $x_T$ are new,
$$
  J_N:=|S_{T-1}|=D_0+N-1\geq N.
$$
Repeated points do not change the span of the evaluation representers,
so $s_{T-1}$ is the power function associated with the $J_N$ distinct points
in $S_{T-1}$.
At time $T-1$, equation~\eqref{eq:optimizer-uncertainty-recursion} with
$\eta=1$, together with $B_0\leq B$ and $C\leq C_B$ from
Remark~\ref{rem:ei-constants}, gives
$$
  \begin{aligned}
    r_T
    &\leq
    \min\left\{
      q r_{T-1},\,
      B_0s_{T-1}(x^\star)
    \right\}
    +Cs_{T-1}(x_T)\\
    &\leq
    B_0s_{T-1}(x^\star)+Cs_{T-1}(x_T)\\
    &\leq
    (B+C_B)\sup_{x\in\X}s_{T-1}(x).
  \end{aligned}
$$
Proposition~\ref{prop:yarotsky-power} therefore gives, for every
$0<b''<1/2$,
$$
  r_T
  \leq
  (B+C_B)A_{b''}
  \exp\{-b''J_N\log(eJ_N)\}.
$$
Since $J_N\geq N$, \eqref{eq:se-one-dimensional-exact-rate} follows with
$A''=(B+C_B)A_{b''}$.
\end{proof}

\section{Measurable policies and minimax bounds}
\label{app:minimax-details}

\subsection{Measurable exact~EI policies}
\label{app:measurable-exact-ei}

\begin{proposition}[Measurable exact~EI policies]
\label{prop:measurable-exact-ei}
There is a measurable exact~EI policy $\underline X=(X_n)_{n\geq1}$ whose
first $n_0$ decisions are the fixed initial design and such that, for every
$f\in\Hk$,
$$
  \EI_n\bigl(X_{n+1}(f)\bigr)
  =
  \max_{x\in\X}\EI_n(x),
  \qquad n\geq n_0.
$$
\end{proposition}

\begin{proof}
For each $n\geq n_0$, we construct an extension of EI to the whole history
space that is Borel measurable in the history and continuous in the candidate
point.  A measurable selection of EI maximizers is a Borel decision rule that
assigns an EI-maximizing candidate point to every history.  The
measurable-selection theorem then gives such rules.  Together with the fixed
initial decisions, they define the required policy.

\emph{Measurability of the posterior quantities.}
Let $H_n:=(\X\times\mathbb R)^n$ for each $n\geq n_0$.  Both $H_n$ and
$\X$ are complete separable
metric spaces.  A decision rule after $n$ evaluations is a Borel map from
$H_n$ to $\X$, even though some histories in $H_n$ cannot occur when the
objective function belongs to $\Hk$ and its values are observed exactly.  We
therefore extend the posterior formulas to all of $H_n\times\X$.

Given $h=((x_1,z_1),\ldots,(x_n,z_n))\in H_n$, put
$\mathbf z_n:=(z_1,\ldots,z_n)^\top$.  The Gram matrix
$K_n=[k(x_i,x_j)]_{i,j=1}^n$ and the best observed value
$m_n=\min_{1\leq i\leq n}z_i$ are continuous functions of the history in
$H_n$.  The map
$(h,x)\mapsto\mathbf k_n(x):=(k(x_i,x))_{i=1}^n$
is continuous on $H_n\times\X$.  For a real matrix $A$ of fixed size,
$$
  A^\dagger
  =
  \lim_{\gamma\downarrow0}
  \bigl(A^\top A+\gamma I\bigr)^{-1}A^\top.
$$
For each $\gamma>0$, the map
$A\mapsto(A^\top A+\gamma I)^{-1}A^\top$ is continuous.  Hence the
Moore--Penrose inverse is Borel measurable as a pointwise limit of continuous
matrix maps.

We use the formulas involving $K_n^\dagger$ in
\eqref{eq:posterior-formulas} to define $\mu_n$ and $s_n$ on
$H_n\times\X$.
For each fixed $x$, $\mu_n(x)$ and $s_n(x)$ are Borel measurable functions of
the history, and for each fixed history they are continuous in $x$.
Thus these formulas define Borel extensions of the posterior mean and
normalized posterior standard deviation.  They agree with the posterior
quantities when the history consists of exact observations of an objective
function $f\in\Hk$.

\emph{Measurability of EI.}
At each history $h\in H_n$, use \eqref{eq:ei-closed-form}, with
$m_n$, $\mu_n$, and $s_n$ given by the extensions above, to define
$x\mapsto\EI_n(x)$.  This agrees with posterior EI on histories arising from
exact observations.  The continuity stated after
\eqref{eq:ei-closed-form} makes $\EI_n$ Borel measurable in the history for
each fixed $x$ and continuous in $x$ for each fixed history.  By the
Carath\'eodory measurability theorem, separability of $\X$ then makes
$(h,x)\mapsto\EI_n(x)$ jointly Borel measurable on $H_n\times\X$.

\emph{Measurable selection of EI maximizers.}
Compactness of $\X$ ensures that $\EI_n$ attains its maximum for every
history.  Applying \citet[Corollary~1]{brownPurves1973} to $-\EI_n$ gives a
Borel map from $H_n$ to $\X$ whose value at each history maximizes $\EI_n$.

\emph{Construction of the policy.}
Together with the fixed initial decisions, these maps define the required
policy $\underline X$.
\end{proof}

\subsection{Deterministic minimax bounds for the squared-exponential kernel}
\label{app:se-minimax-proof}

\paragraph{Covering numbers of a squared-exponential RKHS on the unit cube.}
Given a kernel $k$ on $\R^d$ and a nonempty set
$\mathcal D\subset\R^d$, write $\mathcal H_{k,\mathcal D}$ for the RKHS of
$k|_{\mathcal D\times\mathcal D}$.  For a Hilbert space $\mathcal H$ and
$B\geq0$, write
$$
  \mathcal B_B(\mathcal H)
  :=
  \{h\in\mathcal H:\lVert h\rVert_{\mathcal H}\leq B\}.
$$
When $\mathcal D$ is compact, write $\mathcal C(\mathcal D)$ for the space
of real-valued continuous functions on $\mathcal D$, equipped with the sup
norm.  For $d\geq1$ and $\varrho>0$, let $k_\varrho$ be the
squared-exponential kernel on $\R^d$ written in the form
$$
  k_\varrho(u,v)
  :=
  \exp\biggl(-\frac{\lVert u-v\rVert^2}{2\varrho^2}\biggr),
  \qquad u,v\in\R^d.
$$
Then \citet[Theorem~3]{kuhn2011} proves that the metric entropy
(the logarithm of the sup-norm covering number) of
$\mathcal B_1(\mathcal H_{k_\varrho,[0,1]^d})$ at radius $\varepsilon$ is
of order
$$
  \frac{\{\log(1/\varepsilon)\}^{d+1}}
       {\{\log\log(1/\varepsilon)\}^{d}}
$$
as $\varepsilon\downarrow0$.  We use only the lower estimate.

The next lemma transfers that lower bound to the radius-$B$ ball of
$\mathcal H_{k_\varrho,\X}$.

\begin{lemma}[Metric entropy of the squared-exponential RKHS ball]
\label{lem:se-entropy-lower}
Let $\X\subset\R^d$ be compact with nonempty interior, and fix
$\varrho,B>0$.
Let
$\mathcal N(\mathcal B_B(\mathcal H_{k_\varrho,\X}),\varepsilon,
\lVert\cdot\rVert_\infty)$ be the sup-norm covering number of this ball on
$\X$, and set
$L_\varepsilon:=\log(B/\varepsilon)$.  There are constants $c>0$ and
$\varepsilon_0\in(0,B/(4e))$ such that, for
$0<\varepsilon\leq\varepsilon_0$,
\begin{equation}
  \log\mathcal N(\mathcal B_B(\mathcal H_{k_\varrho,\X}),4\varepsilon,
                  \lVert\cdot\rVert_\infty)
  \geq
  c\,
  \frac{L_\varepsilon^{d+1}}{(\log L_\varepsilon)^d}.
  \label{eq:se-entropy-lower}
\end{equation}
\end{lemma}

\begin{proof}
Choose a closed cube
$$
  \mathcal Q=b_{\mathcal Q}+a_{\mathcal Q}[0,1]^d
  \subset\operatorname{int}(\X),
  \qquad a_{\mathcal Q}>0.
$$
Write $\lVert\cdot\rVert_{\infty,\mathcal Q}$ for the sup norm on
$\mathcal Q$.

The RKHS extension and restriction theorems
\citep[Theorems~10.46--10.47]{wendland2005} show that restriction from $\X$
to $\mathcal Q$ maps $\mathcal B_B(\mathcal H_{k_\varrho,\X})$ onto
$\mathcal B_B(\mathcal H_{k_\varrho,\mathcal Q})$.  Restricting a sup-norm
cover on $\X$ therefore gives a cover on $\mathcal Q$, so
\begin{equation}
  \mathcal N(\mathcal B_B(\mathcal H_{k_\varrho,\mathcal Q}),\varepsilon,
             \lVert\cdot\rVert_{\infty,\mathcal Q})
  \leq
  \mathcal N(\mathcal B_B(\mathcal H_{k_\varrho,\X}),\varepsilon,
             \lVert\cdot\rVert_\infty).
  \label{eq:se-cover-restriction}
\end{equation}

By homogeneity,
\begin{equation}
  \mathcal N(\mathcal B_B(\mathcal H_{k_\varrho,\mathcal Q}),4\varepsilon,
             \lVert\cdot\rVert_{\infty,\mathcal Q})
  =
  \mathcal N(\mathcal B_1(\mathcal H_{k_\varrho,\mathcal Q}),4\varepsilon/B,
             \lVert\cdot\rVert_{\infty,\mathcal Q}).
  \label{eq:se-cover-rescaling}
\end{equation}

The similarity
$$
  \vartheta_{\mathcal Q}(x)
  :=\frac{x-b_{\mathcal Q}}{a_{\mathcal Q}}
$$
maps $\mathcal Q$ onto $[0,1]^d$.  Put
$\varrho':=\varrho/a_{\mathcal Q}$.  Then
$$
  k_\varrho\bigl(\vartheta_{\mathcal Q}^{-1}(u),
                  \vartheta_{\mathcal Q}^{-1}(v)\bigr)
  =
  \exp\biggl(
    -\frac{\lVert u-v\rVert^2}{2(\varrho')^2}
  \biggr)
  =
  k_{\varrho'}(u,v).
$$
The metric entropy estimate in \citet[Theorem~3]{kuhn2011} therefore
applies to $\mathcal H_{k_{\varrho'},[0,1]^d}$.
Composition with $\vartheta_{\mathcal Q}^{-1}$ is an isometric isomorphism
from $\mathcal H_{k_\varrho,\mathcal Q}$ onto
$\mathcal H_{k_{\varrho'},[0,1]^d}$ and preserves the sup norm.
It also maps $\mathcal C(\mathcal Q)$ isometrically onto
$\mathcal C([0,1]^d)$, so arbitrary centers of covers are transported by
composition with
$\vartheta_{\mathcal Q}^{-1}$.  Taking
$\varepsilon':=4\varepsilon/B$, the covering number on the
right-hand side of \eqref{eq:se-cover-rescaling} is therefore the covering
number at radius $\varepsilon'$ of
$\mathcal B_1(\mathcal H_{k_{\varrho'},[0,1]^d})$.  Since
$\log(1/\varepsilon')=L_\varepsilon-\log4$,
\citet[Theorem~3]{kuhn2011} bounds the logarithm of this covering number
below by a positive constant multiple of
$L_\varepsilon^{d+1}/(\log L_\varepsilon)^d$ for all sufficiently small
$\varepsilon$.  Choose $\varepsilon_0\in(0,B/(4e))$ so that this bound holds
whenever $0<\varepsilon\leq\varepsilon_0$.  Combining this estimate with
\eqref{eq:se-cover-restriction}--\eqref{eq:se-cover-rescaling} proves
\eqref{eq:se-entropy-lower}.

\end{proof}

\paragraph{From metric entropy to a lower bound on the number of evaluations.}
Let $\mathcal H$ be an RKHS on $\R^d$ whose reproducing kernel is continuous
on $\R^d\times\R^d$ and bounded above by one on $\X\times\X$, where $\X$ is
nonempty and compact.  Fix $B>0$, and write
$$
  \mathcal B_B(\mathcal H)|_\X
  :=
  \{f|_\X:f\in\mathcal B_B(\mathcal H)\}.
$$
\citet[Theorem~5.1]{xu2024} give constants
$c_0,\varepsilon_1>0$ such that, for every $T\geq1$ and
$0<\varepsilon\leq\varepsilon_1$, any algorithm
$\underline X\in\mathcal A_\infty$ satisfying
$$
  \sup_{g\in\mathcal B_B(\mathcal H)|_\X}
  \biggl\{
    \min_{1\leq j\leq T}g\bigl(X_j(g)\bigr)
    -
    \min_{x\in\X}g(x)
  \biggr\}
  \leq
  \varepsilon
$$
must satisfy
\begin{equation}
  T
  \geq
  c_0
  \frac{
    \log\mathcal N(\mathcal B_B(\mathcal H)|_\X,4\varepsilon,
                   \lVert\cdot\rVert_\infty)
  }{
    \log(B/\varepsilon)
  }.
  \label{eq:xu-query-lower}
\end{equation}
Here $c_0$ can be chosen independently of $\mathcal H$, $B$, and $\X$,
whereas the threshold $\varepsilon_1$ may depend on them.

\begin{lemma}[Squared-exponential evaluation lower bound]
\label{lem:se-query-lower}
Under the assumptions and notation of
Lemma~\ref{lem:se-entropy-lower}, there are constants
$c',c''>0$ and $\varepsilon_*\in(0,\varepsilon_0]$ such that, for every
$T\geq1$ and $0<\varepsilon\leq\varepsilon_*$, any
$\underline X_T\in\mathcal A_T$ whose best observed value has worst-case
simple regret at most $\varepsilon$ over
$\mathcal B_B(\mathcal H_{k_\varrho,\X})$ satisfies
\begin{equation}
  T
  \geq
  c'\,
  \frac{
    \log\mathcal N(\mathcal B_B(\mathcal H_{k_\varrho,\X}),4\varepsilon,
                   \lVert\cdot\rVert_\infty)
  }{
    L_\varepsilon
  }
  \geq
  c''
  \biggl(
    \frac{L_\varepsilon}{\log L_\varepsilon}
  \biggr)^d.
  \label{eq:se-query-lower}
\end{equation}
\end{lemma}

\begin{proof}
The RKHS restriction theorem gives
\begin{equation}
  \mathcal B_B(\mathcal H_{k_\varrho,\X})
  =
  \biggl\{
    \widetilde f|_\X:
    \widetilde f\in\mathcal H_{k_\varrho,\R^d},\
    \lVert\widetilde f\rVert_{\mathcal H_{k_\varrho,\R^d}}\leq B
  \biggr\}
  =
  \mathcal B_B(\mathcal H_{k_\varrho,\R^d})|_\X.
  \label{eq:se-rkhs-ball-restriction}
\end{equation}
Extend $\underline X_T$ to an element of $\mathcal A_\infty$ by querying a
fixed point at every subsequent step.  This leaves the first $T$ query
points and the best value observed among them unchanged.  The kernel
$k_\varrho$ is continuous and bounded above by one on $\X\times\X$.  Hence
both \eqref{eq:xu-query-lower} and
Lemma~\ref{lem:se-entropy-lower} apply when
$0<\varepsilon\leq\varepsilon_*$, where
$$
  \varepsilon_*:=\min\{\varepsilon_0,\varepsilon_1\}.
$$
Taking $c':=c_0$ and $c'':=c'c$ proves both inequalities in
\eqref{eq:se-query-lower}.

\end{proof}

\begin{proof}[Proof of Theorem~\ref{thm:se-minimax}]
\emph{Lower bound.}
Let $\varrho$ be the lengthscale of $k$.  Then
$\Hk=\mathcal H_{k_\varrho,\X}$, so the ball in
Lemma~\ref{lem:se-entropy-lower} is the ball used in
Theorem~\ref{thm:se-minimax}.  Let $c''$ be as in
Lemma~\ref{lem:se-query-lower}, and let $\varepsilon_*$ be its threshold.
Choose $C_1>0$ such that
$$
  c''\biggl(\frac{C_1}{3}\biggr)^d>2
$$
and set
$$
  \varepsilon_N
  :=
  B\exp\{-C_1 N^{1/d}\log(eN)\}.
$$
Then
$$
  L_{\varepsilon_N}
  =
  C_1 N^{1/d}\log(eN),
  \qquad
  \log L_{\varepsilon_N}
  \leq
  3\log(eN)
$$
for all sufficiently large $N$.

Fix a deterministic method with query rules
$\underline X_N=(X_1,\ldots,X_N)$ and Borel recommendation
$\widehat X_N$.  Append $\widehat X_N$ as the $(N+1)$st query.  The resulting
algorithm $\underline X'_{N+1}\in\mathcal A_{N+1}$ is given by
$$
  X'_j:=X_j\quad(1\leq j\leq N),
  \qquad
  X'_{N+1}:=\widehat X_N.
$$
For every $f\in\mathcal B_B(\mathcal H_{k_\varrho,\X})$,
\begin{equation}
  \min_{1\leq j\leq N+1}f\bigl(X'_j(f)\bigr)-\min_{x\in\X}f(x)
  \leq
  f\bigl(\widehat x_N(f)\bigr)-\min_{x\in\X}f(x).
  \label{eq:recommendation-as-query}
\end{equation}
Since $\varepsilon_N\to0$, one has
$\varepsilon_N\leq\varepsilon_*$ for all sufficiently large $N$.
If the method had worst-case loss at most
$\varepsilon_N$, Lemma~\ref{lem:se-query-lower}, applied with $T=N+1$,
would give
$$
  N+1
  \geq
  c''
  \biggl(
    \frac{L_{\varepsilon_N}}
         {\log L_{\varepsilon_N}}
  \biggr)^d
  \geq
  c''\biggl(\frac{C_1}{3}\biggr)^dN
  >
  2N
$$
for all sufficiently large $N$.  This contradicts $N+1\leq2N$.
Thus every deterministic method has worst-case loss greater
than $\varepsilon_N$.  Taking the infimum over the methods gives
$$
  \mathcal R_N^{\mathrm{det}}(B;k,\X)
  \geq
  B\exp\{-C_1 N^{1/d}\log(eN)\},
$$
which is the lower bound in~\eqref{eq:se-minimax-scale} with $c_1=B$.

\emph{Upper bound.}
Consider the method that uses the first $N$ decision rules of
$\underline X^{\mathrm{EI}}$ and recommends the earliest query point at
which the best observed value is attained.  Its worst-case loss is
$R_N^{\mathrm{EI}}(B;k,\X)$, so
$$
  \mathcal R_N^{\mathrm{det}}(B;k,\X)
  \leq
  R_N^{\mathrm{EI}}(B;k,\X).
$$
Theorem~\ref{thm:se-rate}, applied with post-initial budget $N-n_0$, gives
$$
  R_N^{\mathrm{EI}}(B;k,\X)
  \leq
  C_2\exp\{-c_2N^{1/d}\log(eN)\}
$$
for some $C_2,c_2>0$, using \eqref{eq:se-rate} when $d\geq2$ and
\eqref{eq:se-one-dimensional-exact-rate} when $d=1$.  Since $n_0$ does not
depend on $N$, it changes only the constants.  This proves the upper bound in
\eqref{eq:se-minimax-scale}.
\end{proof}

\section{Known continuous prior means}
\label{app:known-means}

\paragraph{Setting and posterior formulas.}

Let $\mu_0\in\mathcal C(\X)$ be a known prior mean, and consider the fixed prior
$\operatorname{GP}(\mu_0,\sigma^2k)$.  Let $f_0:=f-\mu_0$ and assume
$$
  f_0\in\Hk,
  \qquad
  \lVert f_0\rVert_{\Hk}\leq B.
$$
The finite-budget bounds in Theorem~\ref{thm:order-statistic-transfer} and
the upper bounds in Theorems~\ref{thm:matern-rate} and~\ref{thm:se-rate}
hold along every EI trajectory based on this prior and
satisfying~\eqref{eq:approx-ei}.  The sharper one-dimensional
squared-exponential rate requires exact maximization of EI.

For $n\geq n_0$, let
$$
  \mathbf f_{0,n}
  :=
  \mathbf z_n
  -\bigl(\mu_0(x_1),\ldots,\mu_0(x_n)\bigr)^\top
  =\mathcal E_nf_0.
$$
Then
$$
  \begin{aligned}
  \mu_n(x)
  &=\mu_0(x)+\mathbf k_n(x)^\top K_n^\dagger\mathbf f_{0,n}
    =\mu_0(x)+(\Pi_nf_0)(x),
    \qquad x\in\X,\\
  f-\mu_n
  &=(I-\Pi_n)f_0.
\end{aligned}
$$
The Moore--Penrose expression remains valid with repeated query points because
$\mathbf f_{0,n}\in\operatorname{Range}(K_n)$.  For a given sequence of query
points, $\mu_0$ affects the posterior mean but not the posterior covariance or
$s_n$, although it may change the points selected by EI.  The definitions of
$m_n$ and $\EI_n$ remain unchanged.
Continuity of $k$ and $\mu_0$ implies the continuity of
$f,\mu_n,s_n$, and $\EI_n$.  Compactness of $\X$ then gives a minimizer of
$f$ and a maximizer of $\EI_n$.

\paragraph{EI comparison and finite-budget bounds.}

After the initial design,
$\widehat f_{0,n_0}:=\Pi_{n_0}f_0$ is the minimum-norm interpolant of the
centered initial data.  Since
$\lVert\widehat f_{0,n_0}\rVert_{\Hk}\leq
\lVert f_0\rVert_{\Hk}\leq B$, define
$$
  B_0^{(f_0)}
  :=
  \left(B^2-\lVert\widehat f_{0,n_0}\rVert_{\Hk}^2\right)^{1/2}.
$$

Since $V_{n_0}\subseteq V_n$, orthogonal projection and the definition of
$B_0^{(f_0)}$ give, for $n\geq n_0$,
$$
  \lVert(I-\Pi_n)f_0\rVert_{\Hk}
  \leq\lVert(I-\Pi_{n_0})f_0\rVert_{\Hk}
  \leq B_0^{(f_0)}.
$$
The reproducing property, the orthogonality of $(I-\Pi_n)f_0$ to $V_n$, and
the Cauchy--Schwarz inequality then give
\begin{equation}
  |f(x)-\mu_n(x)|
  =\bigl|[(I-\Pi_n)f_0](x)\bigr|
  \leq B_0^{(f_0)}s_n(x),
  \qquad x\in\X.
  \label{eq:known-mean-interpolation-error}
\end{equation}

With the constant $\eta\in(0,1]$ from~\eqref{eq:approx-ei}, define
$$
  a:=\frac{B_0^{(f_0)}}{\sigma},
  \qquad
  q:=1-\eta\kappa_a\in[0,1),
  \qquad
  C:=\sigma\tau(a)>0,
$$
where $\tau(z)=z\Phi(z)+\phi(z)$ and $\kappa_a$ is defined by
\eqref{eq:ei-comparison-constant}.  The proof of
Lemma~\ref{lem:ei-comparison}, with
\eqref{eq:known-mean-interpolation-error} in place of
\eqref{eq:rkhs-error}, gives, for every $n\geq n_0$ and $x\in\X$,
\begin{equation}
  \max\bigl\{
    I_n(x)-B_0^{(f_0)}s_n(x),\,
    \kappa_a I_n(x)
  \bigr\}
  \leq\EI_n(x)
  \leq I_n(x)+Cs_n(x).
  \label{eq:known-mean-ei-comparison}
\end{equation}
Consider an EI trajectory satisfying~\eqref{eq:approx-ei}, and let
$x^\star$ be a global minimizer.  Since
$I_n(x^\star)=r_n$ and
$I_n(x_{n+1})=r_n-r_{n+1}$, the proof of
Lemma~\ref{lem:optimizer-uncertainty}, with $B_0^{(f_0)}$ in place of $B_0$,
gives
\begin{equation}
  r_{n+1}
  \leq
  \min\left\{
    q r_n,\,
    (1-\eta)r_n+\eta B_0^{(f_0)}s_n(x^\star)
  \right\}
  +Cs_n(x_{n+1}).
  \label{eq:known-mean-optimizer-uncertainty-recursion}
\end{equation}
In particular,
\begin{equation}
  r_{n+1}
  \leq q r_n+Cs_n(x_{n+1}).
  \label{eq:known-mean-one-step}
\end{equation}
By~\eqref{eq:selected-point-innovation},
$s_n(x_{n+1})=v_{n+1}$, so
Theorem~\ref{thm:order-statistic-transfer} gives the finite-budget bounds.
Since $k_0$ is continuous,
Lemma~\ref{lem:delta-vanishes} and
Corollary~\ref{cor:qualitative-consistency} also give $r_n\to0$.

\paragraph{Kernel-specific rates.}

Lemma~\ref{lem:residual-envelope-domination} transfers the separation-radius
estimates in Section~\ref{sec:kernels} to $k_0$, yielding the Mat\'ern and
squared-exponential rates under~\eqref{eq:approx-ei}.  For the fixed prior
mean $\mu_0$, these rates hold uniformly over
$$
  \mu_0+\{h\in\Hk,\,\lVert h\rVert_{\Hk}\leq B\}.
$$
Indeed, $B_0^{(f_0)}\leq B$, so Remark~\ref{rem:ei-constants} gives
$q\leq q_B<1$ and $C\leq C_B$.  The diagonal normalization
\eqref{eq:normalization} gives $\lVert f_0\rVert_\infty\leq B$, and hence
$$
  r_{n_0}
  \leq
  2B+\sup_{x\in\X}\mu_0(x)-\inf_{x\in\X}\mu_0(x).
$$

\paragraph{Exact one-dimensional squared-exponential rate.}

When $d=1$, $k$ is squared-exponential, and EI is maximized exactly, the rate
in~\eqref{eq:se-one-dimensional-exact-rate} also holds for a known continuous
prior mean.  Interpolation gives $\EI_n(x)=0$ at queried points, whereas strict
positive definiteness makes the predictive distribution nondegenerate and
$\EI_n(x)>0$ at unqueried points.  Thus exact EI selects a new point whenever
one remains.

For $T>n_0$, let $N:=T-n_0$.  If every point of $\X$ has been queried by
time $T-1$, then $r_T=0$.  Otherwise, the first $T-1$ queries contain at least
$N$ distinct points.  With $\eta=1$, the branch
$(1-\eta)r_{T-1}+\eta B_0^{(f_0)}s_{T-1}(x^\star)$ in
\eqref{eq:known-mean-optimizer-uncertainty-recursion}, together with the bounds
$B_0^{(f_0)}\leq B$ and $C\leq C_B$, gives
$$
  r_T
  \leq B s_{T-1}(x^\star)+C_B s_{T-1}(x_T)
  \leq(B+C_B)\sup_{x\in\X}s_{T-1}(x).
$$
Removing repetitions leaves $s_{T-1}$ unchanged.  The power function based on
any $N$ of the distinct points is an upper bound for $s_{T-1}$.
Proposition~\ref{prop:yarotsky-power} therefore gives, for every
$0<b''<1/2$,
$$
  r_T
  \leq
  (B+C_B)A_{b''}\exp\{-b''N\log(eN)\},
$$
which proves~\eqref{eq:se-one-dimensional-exact-rate} uniformly on the
translated ball.

\begingroup
\raggedright
\bibliographystyle{abbrvnat}
\bibliography{refs}
\endgroup

\end{document}